%% file: LUD_arxiv_revised.tex
\documentclass[a4paper,11pt]{article}

\input{rrp-macro-file}

\usepackage{tikz}
\usepackage{authblk}
\usepackage[moderate]{savetrees}
\tikzset{mynode/.style={inner sep=1pt,fill,outer sep=0,circle}}
\usepackage[english]{babel}
\usepackage[margin=1in]{geometry}
\usepackage{multirow}
\usepackage{amssymb}
\usepackage{amsmath}
\usepackage{subfig}
\usepackage{graphicx}
\usepackage{amsthm}
\usepackage[toc,page]{appendix}
\usepackage{xcolor}
\DeclareMathOperator{\spann}{Span}
\begin{document}

\newtheoremstyle{mystyle}
	{}
	{}
	{\itshape}
	{}
	{\bfseries}
	{.}
	{ }
	{}%
	\theoremstyle{mystyle}

\newtheorem{innercustomgeneric}{\customgenericname}
\providecommand{\customgenericname}{}
\newcommand{\newcustomtheorem}[2]{%
  \newenvironment{#1}[1]
  {%
   \renewcommand\customgenericname{#2}%
   \renewcommand\theinnercustomgeneric{##1}%
   \innercustomgeneric
  }
  {\endinnercustomgeneric}
}

\newcustomtheorem{customthm}{Theorem}
\newcustomtheorem{customlemma}{Lemma}
\newcustomtheorem{customcorollary}{Corollary}

\newtheorem{theorem}{Theorem}[section]
\newtheorem{lemma}[theorem]{Lemma}
\newtheorem{proposition}[theorem]{Proposition}
\newtheorem{definition}[theorem]{Definition}
\newtheorem{corollary}[theorem]{Corollary}
\newcommand{\myparallel}{{\mkern3mu\vphantom{\perp}\vrule depth 0pt\mkern2mu\vrule depth 0pt\mkern3mu}}
\author[1]{Gilad Lerman}
\author[1]{Yunpeng Shi}
\author[2]{Teng Zhang}
\affil[1]{School of Mathematics, University of Minnesota}
\affil[2]{Department of Mathematics, University of Central Florida\authorcr {{\tt \{lerman, shixx517\}@umn.edu}, \tt Teng.Zhang@ucf.edu}\vspace{1.5ex}}

\setcounter{Maxaffil}{0}
\renewcommand\Affilfont{\small}
\title{Exact Camera Location Recovery by Least Unsquared Deviations\thanks{This work was supported by NSF award DMS-14-18386.  We are grateful for the anonymous reviewers and the action editor for the careful reading of the manuscript and the useful suggestions.}}
\date{}
\maketitle
\begin{abstract}We establish exact recovery for the Least Unsquared Deviations (LUD) algorithm of \"{O}zyesil and Singer. More precisely, we show that for sufficiently many cameras
with given corrupted pairwise directions, where both camera locations and pairwise directions are generated by a special probabilistic model, the LUD algorithm exactly recovers the camera locations with high probability. A similar exact  recovery guarantee for camera locations was established for the ShapeFit algorithm by Hand, Lee and Voroninski, but with typically less corruption.
\end{abstract}
\section{Introduction}
The Structure from Motion (SfM) problem asks to recover the 3D structure of an object from its 2D images. These images are taken by many cameras at different orientations and locations. In order to recover the underlying structure, both the orientations and locations of the cameras need to be estimated~\cite{sfmsurvey_2017}.

The common procedure is to first estimate the relative orientations between pairs of cameras from the corresponding essential matrices and then use them to obtain the pairwise directions between cameras \cite{multiviewbook}.  A pairwise direction between two cameras is the normalized vector of their relative location.  The global orientations up to an arbitrary rotation can be concluded via synchronization from the pairwise orientations \cite{Nachimson_LS,ChatterjeeG13_rotation, Govindu04_Lie,HartleyAT11_rotation,MartinecP07_rotation, OzyesilSB15_SDR}. The locations can be derived from the pairwise directions \cite{Nachimson_LS, BrandAT04_LS,  GoldsteinHLVS16_shapekick, Govindu01_LS, Govindu04_Lie,HandLV15,MoulonMM13_Linfty, cvprOzyesilS15,OzyesilSB15_SDR, TronV09_CLS1, TronV14_CLS2}.

This paper mathematically addresses the latter subproblem of estimating global camera locations when given corrupted pairwise directions with missing values. In doing so, it follows the corruption model and the mathematical problem of Hand, Lee and Voroninski (HLV) \cite{HandLV15}, which are described next.

\textbf{The HLV model:} Assume $n$ cameras, indexed by $[n]=\{1,2,\dots, n\}$, with locations $\bt^*_1, \ldots, \bt^*_n \subset \mathbb{R}^3$, i.i.d.~sampled from $N(\b0,\bI)$.  Let $G([n],E)$ be  drawn from the Erd\"{o}s-R\'{e}nyi ensemble $G(n,p)$ of $n$ vertices with probability of connection $p$. That is, an edge with index $ij\in[n]\times[n]$ is independently drawn between cameras $i$ and $j$ with probability $p$. For any $i$, $j\in [n]$, $ij$ and $ji$ appear at most once in the index set of edges $E$ so that there is no repetition. For each edge with index $ij\in E$, a possibly corrupted pairwise direction vector $\bga_{ij}\in S^2$ is assigned. More precisely, $E$ is partitioned into sets of ``good" and ``bad" edges, $\Eg$ and $\Eb$ respectively, and the pairwise direction vectors are obtained in each set as follows: If $ij\in \Eg$, then $\bga_{ij}$ is the ground truth pairwise direction: \begin{equation}\label{eq:gammastar}
\bga_{ij}^*=\frac{\bt_i^*-\bt_j^*}{\|\bt_i^*-\bt_j^*\|,}
 \end{equation}
where $\|\cdot\|$ denotes the Euclidean norm.
 Otherwise, $\{\bga_{ij}\}_{ij\in \Eb}$ are arbitrarily assigned in $S^2$.  
 The level of corruption of the HLV model is quantified by  $\epsilon_b=\frac{1}{n}(\text{maximal degree of } \Eb)$. 
 The parameters of the HLV model are $n$, $p$ and $\epsilon_b$.

\textbf{The HLV problem and its solutions:} Given data sampled from the HLV model and assuming a bound on the corruption parameter $\epsilon_b$, the exact recovery problem is to reconstruct,
up to ambiguous translation and scale, $\{\bt_i^*\}_{i=1}^n$ from $\{\bga_{ij}\}_{ij\in E}$. Hand, Lee and Voroninski addressed this problem while assuming $\epsilon_b=O(p^5/\log^3 n)$ and using their ShapeFit algorithm \cite{HandLV15}. Here we address this problem with the weaker assumption  $\epsilon_b=O(p^{7/3}/\log^{9/2} n)$, while using the LUD algorithm \cite{cvprOzyesilS15}.

\subsection{Previous Works}
In the past two decades, a variety of algorithms have been proposed for estimating global camera locations from corrupted pairwise directions~\cite{sfmsurvey_2017}. The earliest methods use least squares optimization \cite{Nachimson_LS, BrandAT04_LS,  Govindu01_LS} and often result in collapsed solutions. That is, the camera locations are usually wrongly estimated around few points. Constrained Least Squares (CLS) \cite{TronV09_CLS1, TronV14_CLS2} utilizes a least squares formulation with an additional constraint to avoid collapsed solutions. Another least squares solver with anti-collapse constraint is semidefinite relaxation (SDR) \cite{OzyesilSB15_SDR}. Its constraint is non-convex and makes it hard to solve even after convex relaxation. Other non-least-squares solvers include the $L_{\infty}$ method \cite{MoulonMM13_Linfty} and the Lie-Algebraic averaging method~\cite{Govindu04_Lie}. However, all the above methods are sensitive to outliers.

Recently, \"{O}zyesil and Singer~\cite{cvprOzyesilS15} proposed the Least Unsquared Deviation (LUD) algorithm and numerically demonstrated its robustness to outliers and noise. Given the pairwise directions $\{\bga_{ij}\}_{ij\in E}$, the LUD algorithm estimates the camera locations $\{\bt^*_i\}_{i=1}^n$ by $\{\hat\bt_i\}_{i=1}^n \subset \reals^3$, which solve the following constrained optimization problem with the additional parameters $\{\hat \alpha_{ij}\}_{ij\in E} \subset \mathbb{R}$:
\begin{equation}\label{eq:LUD}
(\{\hat{\bt}_i\}_{i=1}^n,\{\hat{\alpha}_{ij}\}_{ij\in E}\!)=\!\!\argmin_{\latop{\{\bt_i\}_{i=1}^n \subset \R^3}{\{\alpha_{ij}\}_{ij \in E} \subset \R}}\!\sum\limits_{ij\in E}\|\bt_i-\bt_j-\alpha_{ij}\bga_{ij}\| \text{  s.t. } \alpha_{ij}\geq 1\text{ and } \sum_i \bt_i=\b0.
\end{equation}
This formulation is very similar to that of CLS, but uses least absolute deviations instead of least squares in order to gain robustness to outliers. Numerical results in \cite{cvprOzyesilS15} demonstrate that LUD can exactly recover the original locations even when some pairwise directions are maliciously corrupted.

Following \"{O}zyesil and Singer, Hand, Lee and Voroninski \cite{HandLV15} proposed the ShapeFit algorithm as a theoretically guaranteed solver.  Given the pairwise directions $\{\bga_{ij}\}_{ij\in E}$, the ShapeFit algorithm estimates the locations $\{\bt_i^*\}_{i=1}^n$ by solving the following convex optimization problem:
\[
\min_{\{\bt_i\}_{i=1}^n \subset \R^3} \sum_{ij\in E}\|P_{\bga_{ij}^\perp}(\bt_i-\bt_j)\| \text{  s.t. } \sum_{ij\in E}\langle \bt_i-\bt_j,\bga_{ij}\rangle=1 \text{ and } \sum_{i=1}^n \bt_i=\b0,
\]
where $P_{\bga_{ij}^\perp}$ denotes the orthogonal projection onto the orthogonal complement of $\bga_{ij}$.

Empirically, for low levels of noise and corruption, ShapeFit is more accurate than LUD. Figure \ref{fig:simu} demonstrates the empirical behavior of ShapeFit and LUD for synthetic data. We remark that in this case of synthetic data, stability can be measured as the magnitude of the rate of change of accuracy with respect to corruption or noise. Figures 1 and 2 of Goldstein et al.~\cite{GoldsteinHLVS16_shapekick} demonstrate similar behavior, but emphasize exact recovery at lower corruption levels, where ShapeFit often outperforms LUD. Practical results are demonstrated in~\cite{GoldsteinHLVS16_shapekick, SenguptaAGGJSB17,AAB} and seem to indicate similar behavior.
Most notably, LUD is more stable, where stability for real data sets is demonstrated by consistent performance of different simulations for the same data set as well as consistent performance among different data sets.

We are unaware of any careful explanation of the differences between the performance of LUD and ShapeFit, which are demonstrated in Figure \ref{fig:simu}.
To address this issue, we note that the LUD constraints are $\alpha_{ij}\geq 1$ for all $ij\in E$, where each $\alpha_{ij}$ is a relaxation of $\|\bt_i-\bt_j\|$. These constraints force the nearby locations to be sufficiently separated. In other words, short edges are extended to prevent collapsed solutions. In contrast, since the constraint $\sum_{ij\in E}\langle \bt_i-\bt_j,\bga_{ij}\rangle=1$ of ShapeFit only fixes the global scale instead of restricting the length of each edge, it cannot avoid collapse of the whole graph into several clusters. Therefore, under high levels of corruption and noise, where a possible collapse is a major concern, LUD is more accurate and stable. However, under low levels of corruption and noise, the extension of short edges mentioned above may deform the solution of LUD and result in inaccurate estimation. We remark that similarly to the extension of short edges, \cite{BATA} discusses the shrinkage of long edges by LUD. However, \cite{BATA}, which only experiments with low levels of corruption, wrongly claims that ShapeFit is generally superior to LUD.

Some recent works seek to further improve or utilize LUD and ShapeFit. Goldstein et al.~\cite{GoldsteinHLVS16_shapekick}  presents an accelerated version of ShapeFit using ADMM. However, it sacrifices accuracy for speed. Sengupta et al.~\cite{SenguptaAGGJSB17}  presents a novel heuristic for estimating the fundamental matrices with rank constraints, which directly relies on LUD. Zhuang et al.~\cite{BATA} proposed an angle-based formulation to address the unreasonable high weights of long-edge terms in LUD and ShapeFit. However, both~\cite{SenguptaAGGJSB17} and~\cite{BATA} rely on good initializations and lack recovery and convergence guarantees. Other works seek to detect and remove corrupted pairwise directions as a preprocessing step for common camera location solvers, in particular, for LUD and ShapeFit. Wilson and Snavely~\cite{1dsfm} proposed the 1DSfM algorithm for identifying outlying pairwise directions. It projects the 3D locations and pairwise directions to 1D and solves an ordering problem using a heuristic method. However, this method suffers from convergence to local minima. Furthermore, the projection to 1D loses information. Shi and Lerman~\cite{AAB} proposed the All-About-that-Base (AAB) algorithm for separating corrupted and uncorrupted pairwise directions. They established a near-perfect separation guarantee for a basic version of this algorithm. They demonstrated state-of-the-art numerical results, where the most competitive procedure in their real data experiments was LUD preprocessed by AAB.

\begin{figure}[h!]
	\centering
	\subfloat{\includegraphics[width=.45\textwidth]{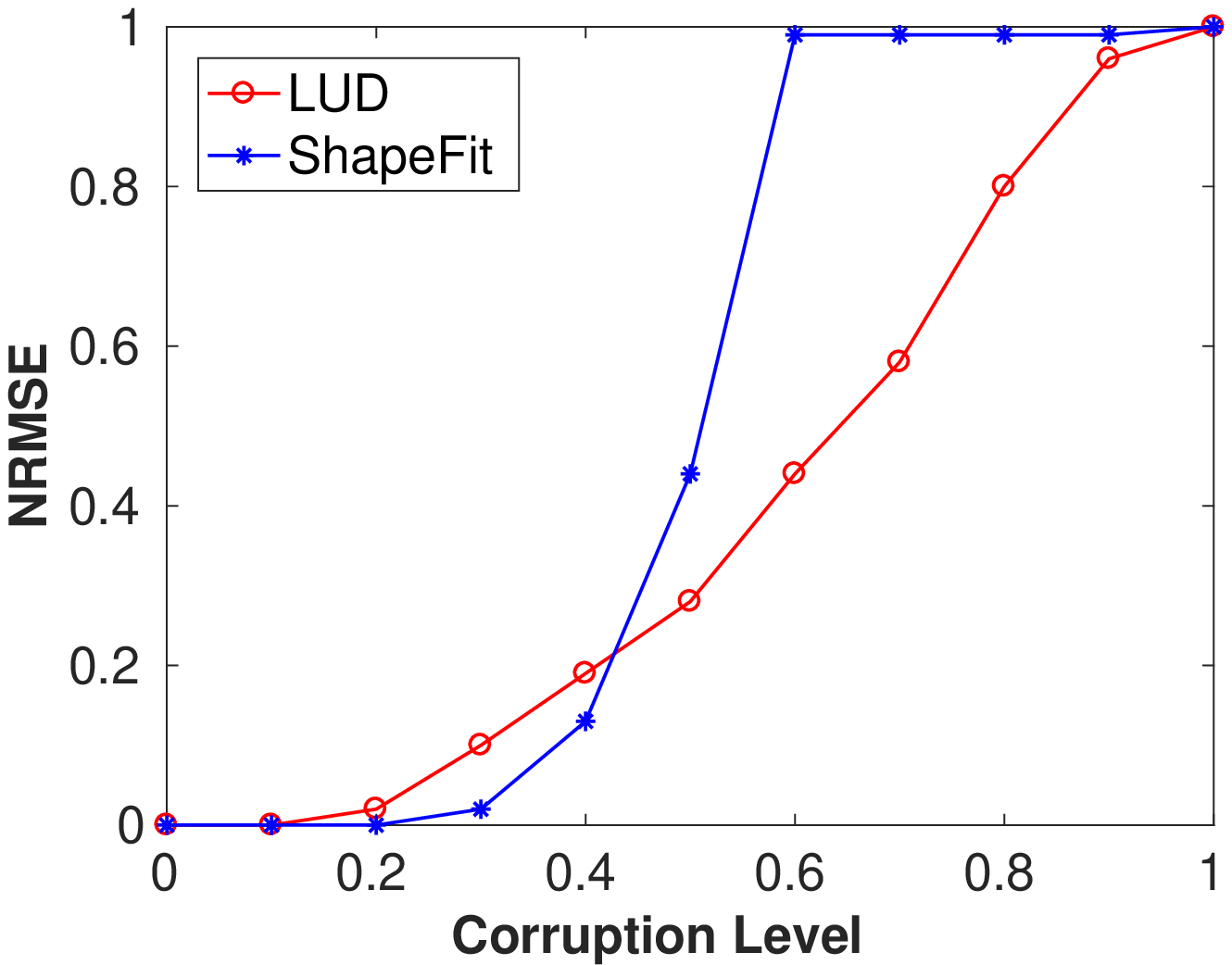}}
	\subfloat{\includegraphics[width=.45\textwidth]{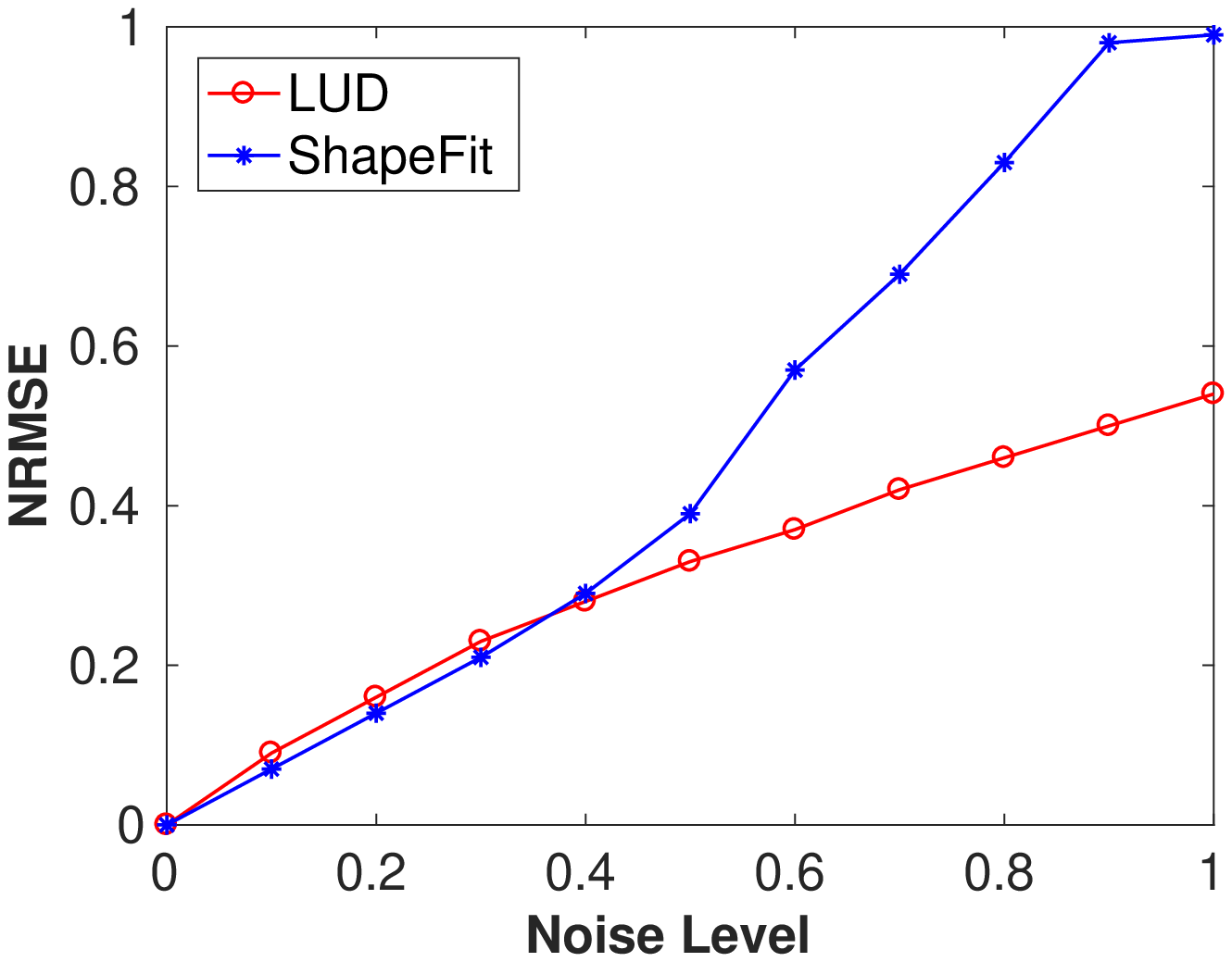}}
	\caption{{Empirical performance of LUD and ShapeFit under corruption and noise for synthetic data. Both methods are implemented using the CVX-SDPT3 package. Left: Data is generated by the HLV model with $n=50$ and $p=0.5$. The corruption level is measured by $|\Eb|/|E|$ instead of $\epsilon_b$ and takes values in $[0,1]$. Right: The ground truth is generated by the HLV model with $n=50$, $p=0.5$ and $\Eb = \emptyset$. For each $ij\in E$, $\bga_{ij}=(\bga_{ij}^*+\sigma \bv_{ij})/\| \bga_{ij}^*+\sigma \bv_{ij} \|$, where $\bv_{ij}$ is uniformly distributed on $S^2$ and $0 \leq \sigma \leq 1$ is the noise level.
In both figures the performance is measured by the normalized root mean squared error (NRMSE):
$\text{NRMSE}^2= {{\sum_{i=1}^n \|\kappa^*\hat \bt_i-\bt_i^*\|^2}}/{\sum_{i=1}^n\|\bt_i^*\|^2}$, where $\kappa^*=\argmin_{\kappa\in\mathbb{R}}\sum_{ij\in E} \|\kappa\hat\bt_i-\bt_i^*\|^2$.}
\label{fig:simu}}
\end{figure}

The mathematical problem discussed in this paper is an example of a convex recovery problem. Other such problems include, for example, recovering sparse signals, low-dimensional signals and underlying subspaces.
There seem to be two different kinds of theoretical guarantees for convex recovery problems. Guarantees of the first kind construct dual certificates~\cite{CandesLMW11_robustpca,CandesT05_decode,ChandrasekaranSPW11}. Guarantees of the second kind show that the underlying object is the minimizer of the convex objective function, and it is sufficient to show this in a small local neighborhood~\cite{CoudronL12_reaper,LMTZ2014,ravikumar2011,XuCS12_robustpca,ZhangL14_novel}.
The latter guarantees often require geometric methods.
It is evident from page 33 of \cite{HandLV15} that the guarantees of ShapeFit are of the second kind. Nevertheless, the graph-theoretic approach of \cite{HandLV15} is completely innovative and enlightening. In particular, it clarifies the effect of vertex perturbation on edge deformation.

\subsection{This Work}
This paper proves exact recovery of LUD under the HLV model up to ambiguous scale and translation.
More precisely, it establishes the following theorem.
\begin{customthm}{1}\label{thm:main}
There exist absolute constants $n_0$,  $C_0$ and $C_1$ such that for $n > n_0$ and for $\{\bt_i^*\}_{i=1}^n\subseteq \mathbb{R}^3$, $E \subseteq [n] \times [n]$ and $\{\bga_{ij}\}_{ij\in E}\subseteq\mathbb{R}^3$ generated by the HLV model with parameters $n$, $p$ and $\epsilon_b$ satisfying $C_0 n^{-1/3}\log^{1/3} n\leq p\leq 1$ and $\epsilon_b\leq C_1p^{7/3}/\log^{9/2} n$, LUD recovers $\{\bt_i^*\}_{i=1}^n$ up to translation and scale with probability at least $1-1/n^4$.
\end{customthm}
To the best of our knowledge this theorem is the first exact recovery result for LUD under a corrupted model. Theorem 1.2 of Hand, Lee and Voroninski \cite{HandLV15} provides exact recovery for ShapeFit under the same model. Both theorems restrict the minimal value of $p$ and the maximal degree of corruption $\epsilon_b$.
Typically, Theorem~\ref{thm:main} tolerates more corruption. Indeed, the higher the upper bound on $\epsilon_b$, the higher the corruption that the algorithm can tolerate. Theorem 1.2 of \cite{HandLV15} requires a bound of order $O(p^5/\log^3 n)$ and Theorem \ref{thm:main} requires a bound of order $O(p^{7/3}/\log^{9/2}n)$.
Therefore in sparse settings where $p \ll 1$, e.g., $p \approx n^{-\alpha}$, Theorem \ref{thm:main} guarantees recovery with more corruption than Theorem 1.2 of \cite{HandLV15}.

There are two additional differences between the theorems, which we find minor. First, in Theorem 1.2 of \cite{HandLV15} the lower bound on $p$ is of order $n^{-1/2}\log^{1/2} n$. While our lower bound is of order
$n^{-1/3}\log^{1/3} n$, it can be modified to be of order $n^{\delta-1/2}\log^{1/2-\delta} n$ for any positive $\delta$ sufficiently small, however, the multiplying constant, $C_0$ depends on $\delta$ and explodes as $\delta$ approaches zero.  The second difference is that  Theorem 1.2 of \cite{HandLV15} was extended to Euclidean spaces with sufficiently high dimensions  (see Theorem 1.1 of \cite{HandLV15}). We can easily extend Theorem~\ref{thm:main} to any fixed higher dimension, though we are not sure about the case where both the dimension and number of locations increase to infinity. Nevertheless, we would rather focus on the three-dimensional case because of the motivating problem from computer vision.

We remark that our analysis borrows various ideas from the work of Hand, Lee and Voroninski \cite{HandLV15}. In fact, we find it interesting to show that their innovative and nontrivial ideas are not limited to a specific objective function, but can be extended to another one.

The main ideas of the proof of Theorem~\ref{thm:main} are discussed in Section~\ref{sec:proof_main}, while additional technical details are left to other sections.
The novelties of this work are emphasized in Section~\ref{sec:nov}.

\section{Proof of Theorem~\ref{thm:main}}
\label{sec:proof_main}
\begin{figure}[h!]
	\centering
	\includegraphics[width=1\textwidth]{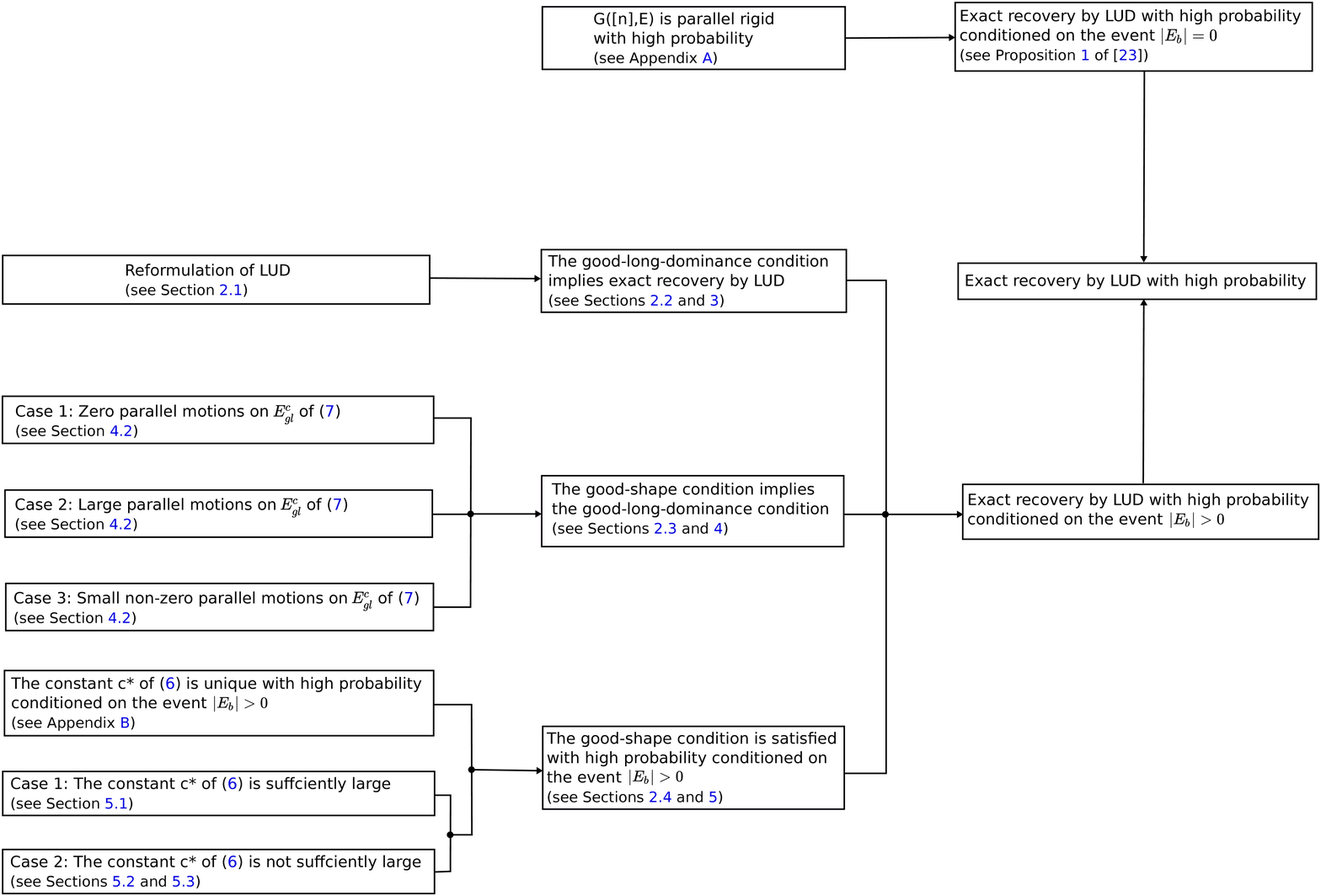}
	\caption{Roadmap for the proof of Theorem~\ref{thm:main}.}
\label{fig:map}
\end{figure}
Figure~\ref{fig:map} presents a roadmap for the proof of Theorem~\ref{thm:main}. The organization of the paper can be described according to a more simplistic version of this roadmap.
Section \ref{sec:oracle} reformulates the LUD problem. Section \ref{sec:goodlong} uses the new formulation to define the ``good-long-dominance condition'' and states that under this condition LUD exactly recovers $\{\bt_i^*\}_{i=1}^n$. Section \ref{sec:goodshape} defines the ``good-shape condition'' and claims that it implies the good-long-dominance condition. Section \ref{sec:ver} shows that under the HLV model the good-shape condition is satisfied with high probability and thus concludes the proof of the theorem. At last, Section~\ref{sec:nov} discusses the novelties in our proof. Details of proofs of the main results of this section are left to Sections \ref{sec:proof1}-\ref{sec:proof3} and the Appendix.

We make the above description more precise so it reflects the roadmap of Figure~\ref{fig:map}.
Our proof of Theorem~\ref{thm:main} assumes that $|\Eb|>0$, where $|E_b|$ denotes the number of elements in $E_b$.
Under the setting of Theorem~\ref{thm:main}, this assumption is sufficient to conclude the theorem. Indeed, Proposition 1 of \cite{cvprOzyesilS15} implies that if $|\Eb| = 0$ and the underlying graph is parallel rigid, then LUD recovers the true solution $\{\bt_i^*\}_{i=1}^n$ up to translation and scale. Appendix~\ref{sec:parallel} reviews this notion of parallel rigidity and shows that under the setting of Theorem \ref{thm:main}, the generated graph is parallel rigid with high probability. Consequently, exact recovery by LUD occurs with high probability when $|\Eb|=0$ and thus it is sufficient to study the case where $|\Eb|>0$.

A technical notion that is crucial in understanding the roadmap is the scale $c^*$ obtained by LUD with respect to the ground truth solution. 
More precisely, when LUD recovers the ground truth locations $\{\bt_i^*\}_{i=1}^n$, it outputs the scaled and shifted locations $\{c^*\bt_i^*+\bt_s\}_{i=1}^n$.
The constant $c^*$ is used to define the notion of good and long edges, 
which is further used to define the above mentioned notions of good-long-dominance and good-shape conditions. To make these notions well-defined, $c^*$ has to be unique. Appendix \ref{sec:cstar} shows that under the setting of Theorem \ref{thm:main} and the sufficient assumption $|\Eb| > 0$, $c^*$ is unique with high probability.
The three and two cases specified in the left hand side of Figure~\ref{fig:map}, which use the constant  $c^*$ and the set of good and long edges, $\Egl$, will be later clarified in Sections \ref{sec:proof1}-\ref{sec:proof3}.

In Sections \ref{sec:goodlong}, \ref{sec:goodshape}, \ref{sec:proof1}, \ref{sec:proof2} and part of Appendix \ref{sec:cstar}, the setting is deterministic.
It assumes a graph $G([n],E)$ with distinct ground truth locations $\{\bt_i^*\}_{i=1}^n$. It also assumes that $E$ is partitioned into $E_b$ and $E_g$. For $ij \in E_g$, the pairwise direction $\bga_{ij}$ is $\bga_{ij}^*$ of \eqref{eq:gammastar} and for $ij \in E_b$, $\bga_{ij}$ is arbitrarily assigned. Except for Appendix \ref{sec:cstar}, this deterministic setting also assumes that $c^*$ is unique. We remark that the latter requirement or other requirements in these sections and appendix, such as the good-long-dominance condition, good-shape condition or non-self-consistency, may restrict the topology of $G([n],E)$, the vertex locations and the corrupted edges.

Throughout the paper we pursue the following conventions and assumptions.
For brevity, we say that an event in our setting holds with overwhelming probability if its probability is at least $1-e^{-C n^{\alpha}}$ for some $\alpha$, $C>0$.
We remark that while the paper has many probabilistic estimates, $p$ is reserved for the connection probability of the HLV model. We often refer to ``locations $\{\bt_i\}_{i=1}^n$'', even though $\{\bt_i\}_{i=1}^n$ is the set of locations. Similarly, we write ``pairwise directions $\{\bga_{ij}\}_{ij\in E}$''. We sometimes refer to the set of vertex locations by $T$. 
Whenever we talk about ground truth camera locations, we assume they are distinct even if we do not specify this. We denote vectors by boldface lower-case letters and matrices by boldface upper-case letters.


\subsection{Reformulation of the Problem}\label{sec:oracle}
We suggest an equivalent formulation of the LUD optimization problem, which gets rid of the variables $\{\alpha_{ij}\}_{ij\in E}$.
We express the optimal $\alpha_{ij}$ in terms of $\{\hat\bt_i\}_{i=1}^n$ and $\{\bga_{ij}\}_{ij\in E}$ as follows:
\begin{equation}\label{eq:alph}
\hat\alpha_{ij} = \argmin\limits_{\alpha_{ij}\geq 1}\|\hat\bt_i-\hat\bt_j-\alpha_{ij}\bga_{ij}\|.
\end{equation}
Figures \ref{fig:1} and \ref{fig:2} illustrate the value of $\hat\alpha_{ij}$ in two complimentary cases. Note that in both figures, $\hat\alpha_{ij}$ is obtained by minimizing the length of the dashed line.
These figures thus demonstrate the following equivalent expression for $\{\hat\alpha_{ij}\}_{ij\in E}$:
\begin{align*}
\hat\alpha_{ij}=\begin{cases}
 \| P_{\bga_{ij}}(\hat \bt_i-\hat \bt_j)\|, & \text{if }  \langle \bga_{ij}\,, \hat \bt_i-\hat \bt_j\rangle>1;\\
1,&  \text{if }  \langle \bga_{ij}\,, \hat \bt_i-\hat \bt_j\rangle\leq 1,
\end{cases}
\end{align*}
where $P_{\bga_{ij}}$ denotes the orthogonal projection onto $\bga_{ij}$.
\begin{figure}
 \centering
\begin{minipage}[]{0.47\textwidth}
\begin{tikzpicture}[scale=3.5,line width=1pt]
 \coordinate (1) at (0,0);
 \coordinate (2) at (1,0);
 \coordinate (3) at (1.5,0.8);
 \coordinate (4) at (1.5,0);
 \foreach \x in {1,2,...,4}{
     \node[mynode] at (\x) {};
 }
 \draw[->, >=latex] (1) -- (2);
 \draw[->, >=latex] (2) -- (4);
 \draw[->, >=latex] (1) -- (3);
 \draw[thick, dashed] (4) -- (3);
 \node at (0,-0.08) {$0$};
 \node at (1,-0.1) {$\bga_{ij}$};
 \node at (1.5,0.9) {$\hat \bt_i-\hat \bt_j$};
 \node at (1.5,-0.1) {$\hat \alpha_{ij}\bga_{ij}$};
\end{tikzpicture}
\caption{{Demonstration of the choice of $\hat\alpha_{ij}$ when  $ \langle \bga_{ij}\,, \hat \bt_i-\hat \bt_j\rangle>1$.
By definition, $\hat\alpha_{ij}=\| P_{\bga_{ij}}(\hat \bt_i-\hat \bt_j)\|.$ \label{fig:1}}}
\end{minipage}
\hspace{0.5cm}
\begin{minipage}[]{0.47\textwidth}
\begin{tikzpicture}[scale=3.5,line width=1pt]
\coordinate (1) at (0,0);
 \coordinate (2) at (1,0);
 \coordinate (3) at (0.6,0.8);
 \foreach \x in {1,2,...,3}{
     \node[mynode] at (\x) {};
 }
 \draw[->, >=latex] (1) -- (2);
 \draw[->, >=latex] (1) -- (3);
 \draw[thick, dashed] (2) -- (3);
 \node at (0,-0.08) {$0$};
 \node at (1,-0.1) {$\bga_{ij}=\hat \alpha_{ij}\bga_{ij}$};
 \node at (0.6,0.9) {$\hat \bt_i-\hat \bt_j$};
\end{tikzpicture}
\caption{Demonstration of the choice of $\hat\alpha_{ij}$ when $ \langle \bga_{ij}\,, \hat \bt_i-\hat \bt_j\rangle\leq 1$. By the constraint $\hat\alpha_{ij}\geq 1$, $\hat\alpha_{ij}=1.$}\label{fig:2}
\end{minipage}
\end{figure}
Plugging the above optimal values of $\{\hat\alpha_{ij}\}_{ij\in E}$ into \eqref{eq:LUD}, we obtain an equivalent LUD formulation:
\begin{equation}\label{eq:LUD1}
\{\hat{\bt}_i\}_{i=1}^n=\argmin_{\{\bt_i\}_{i=1}^n \subset \mathbb{R}^3} \sum\limits_{ij\in E}f_{ij}(\bt_i\,,\bt_j)\,\,\,\text{ subject to } \sum_{i=1}^n \bt_i=\b0,
\end{equation}
where
\begin{align}
\label{eq:def_f}
f_{ij}(\bt_i,\bt_j)=\begin{cases}
 \| P_{\bga_{ij}^{\perp}}( \bt_i- \bt_j)\|, & \text{if }  \langle \bga_{ij}\,, \hat \bt_i-\hat \bt_j\rangle>1;\\
\| \bt_i- \bt_j-\bga_{ij}\|, & \text{if }  \langle \bga_{ij}\,, \hat \bt_i-\hat \bt_j\rangle\leq 1.
\end{cases}
\end{align}

Our analysis requires formulating an oracle problem that determines the particular shift and scale found by LUD. That is, we assume we know the ground truth solution $\{\bt_i^*\}_{i=1}^n$ and we ask for the scale $c^*$ and shift $\bt_s$ such that $\{c^*\bt_i^*+\bt_s\}_{i=1}^n$ minimizes the LUD problem. This oracle problem is formulated as follows:
\begin{align}\label{eq:oracle}
(c^*,\bt_s)=\argmin_{c\in\mathbb{R}, \bt\in \mathbb{R}^3}\sum\limits_{ij\in E}f_{ij}(\bt_i\,,\bt_j)\,\,\,\text{ subject to } \sum_{i=1}^n \bt_i=\b0 \text{ and } \bt_i=c\bt_i^*+\bt.
\end{align}
We later show in Appendix \ref{sec:cstar} that $c^*$ is unique with overwhelming probability under the setting of Theorem \ref{thm:main} and our assumption that $\Eb \neq \emptyset$. 
The uniqueness of $\bt_s$ follows from the LUD constraint $\sum_i\bt_i=\b0$.
We will prove Theorem \ref{thm:main} by showing that $\hat\bt_i=c^*\bt_i^*+\bt_s$ for all $i\in [n]$.

\subsection{Exact Recovery under the Good-Long-Dominance Condition}\label{sec:goodlong}
We establish the recovery of the ground truth locations $\{\bt_i^*\}_{i=1}^n$ by LUD up to translation and scale under a geometric condition, which we refer to as the good-long-dominance condition. The set of good and long edges, $\Egl$, and its complement are defined by
\begin{equation}\label{eq:E1}
\Egl=\{ij\in \Eg|\text{ }\|\bt_i^*-\bt_j^*\|>1/c^*\} \text{ and }\Egl^c=E\setminus \Egl.
\end{equation}
The sets $\Egl$ and $\Egl^c$ are well-defined if $c^*$ uniquely solves \eqref{eq:oracle}. As explained above, in this and the next section (as well as when providing supplementary details in Sections \ref{sec:proof1} and \ref{sec:proof2}), we assume a ``deterministic setting'', where $c^*$ is unique. On the other hand, when assuming the setting of Theorem \ref{thm:main} and the sufficient condition $|E_b|>0$, $c^*$  is unique with overwhelming probability.

\begin{definition}[Good-Long-Dominance Condition]\label{def:goodlong}
We say that $\{\bt_i^*\}_{i=1}^n$, $E = \Eg \cup \Eb \subseteq[n]\times[n]$ and $\{\bga_{ij}\}_{ij\in E}$  satisfy the good-long-dominance condition if for any perturbation vectors $\{\beps_i\}_{i=1}^n\in\mathbb{R}^3$ such that $ \sum_{i=1}^n \beps_i=\b0$ and $ \sum_{i=1}^n \langle\beps_i, \bt_i^*\rangle =0$,
\begin{align}\label{eq:deterministic0}
\sum_{ij\in \Egl}\|P_{\bga_{ij}^{* \perp}}(\beps_i-\beps_j)\|\geq \sum_{ij\in \Egl^c}\|\beps_i-\beps_j \|.
\end{align}
\end{definition}
In order to clarify this condition, we assume that the variables $\{\bt_i\}_{i=1}^n$ are perturbed by $\{\beps_i\}_{1=1}^n$ respectively from the ground truth $\{c^*\bt_i^*+\bt_s\}_{i=1}^n$. As explained later in \eqref{eq:caseA}, the change in the objective function of \eqref{eq:LUD1}, when restricted to the sum over $\Egl$, is the LHS of \eqref{eq:deterministic0}. Furthermore, as explained later in \eqref{eq:caseB}, the change in the objective function of \eqref{eq:LUD1}, when restricted to $\Egl^c$, is bounded above by the RHS of \eqref{eq:deterministic0}. The condition thus shows that the change in the objective function due to the good and long edges dominates the change due to all other edges.

At last, we formulate the following theorem, which is proved in Section \ref{sec:proof1}.
\begin{customthm}{2}\label{thm:deterministic0}
If $\{\bt_i^*\}_{i=1}^n$, $E = \Eg \cup \Eb \subseteq[n]\times[n]$ and $\{\bga_{ij}\}_{ij\in E}$ satisfy the good-long-dominance condition, then LUD exactly recovers the ground truth solution up to translation and scale. That is, the solution 
of \eqref{eq:LUD1} has the form $\hat\bt_i=c^*\bt_i^*+\bt_s$ for $i\in [n]$, where $c^*$ and $\bt_s$ solve \eqref{eq:oracle}.
\end{customthm}
\subsection{Exact Recovery under the Good-Shape Condition}\label{sec:goodshape}
We show that the good-long-dominance condition is satisfied when the graph $E$ has certain properties. We first review the definitions of the following two properties suggested in~\cite{HandLV15}: a $p$-typical graph and $c$-well distributed vertices.
\begin{definition}
A graph $G([n], E)$ is $p$-typical if it satisfies the following propositions:\\
1. $G$ is connected.\\
2. Each vertex of $G$ has degree between $\frac{1}{2}np$ and $2np$.\\
3. Each pair of vertices has codegree between $\frac{1}{2}np^2$ and $2np^2$, where the codegree of a pair of
vertices $ij$ is defined as $|\{k\in [n] : ik, jk \in E\}|$.
\end{definition}
\begin{definition}
Let $G=G([n],E)$ be a graph and let $T = \{\bt_i\}_{i=1}^n\subseteq \mathbb{R}^3$ be a set of vertex locations. For $\bx$, $\by\in \mathbb{R}^3$, $c>0$ and $A\subseteq T$, we say that $A$ is $c$-well-distributed
with respect to $(\bx, \by)$ if the following holds for any $\bh \in \mathbb{R}^3$:
\[\frac{1}{|A|}\sum_{t\in A} \|P_{\spann\{\bt-\bx,\bt-\by\}^\perp}(\bh)\| \geq c · \|P_{(\bx-\by)^\perp}(\bh)\|.\]
We say that $T$   is $c$-well-distributed along $G$ if for all distinct $1\leq i, j \leq n$, the set $S_{ij} = \{\bt_k \in T: ik, jk \in E(G)\}$ is $c$-well-distributed with respect to $(\bt_i, \bt_j)$.
\end{definition}

Let $K_n$ denote the complete graph with $n$ vertices and $E(K_n)$
denote the set of edges of $K_n$.

Using the above notation and definitions, we formulate a geometric condition on $\Egl$ and $G([n],E)$ that guarantees exact recovery by LUD.
\begin{definition}[Good-Shape Condition]\label{def:goodshape}
Let  $p$, $\beta$, $\epsilon_0$, $\epsilon_1$, $c_1 \in(0,1]$, $c_0 \geq 1$ and $\Egl$ be the set of good-long edges defined above. We say that $\{\bt_i^*\}_{i=1}^n$, $E = \Eg \cup \Eb \subseteq[n]\times[n]$ and $\{\bga_{ij}\}_{ij\in E}$ satisfy the good-shape condition with the parameters
$p$, $\beta$, $\epsilon_0$, $\epsilon_1$, $c_0$, $c_1$, if the following hold:
\begin{enumerate}
\item \label{cond:gs_p_typical} $G$ is $p$-typical.
\item \label{cond:gs_large_angles}For any distinct $ij\in E(K_n)$, there exists at least $n-\epsilon_1 n$ indices $k\neq i,j$ such that $1-\langle\bga_{ij}^*,\bga_{ik}^*\rangle\geq \beta^2$ and  $1-\langle\bga_{ij}^*,\bga_{jk}^*\rangle\geq \beta^2$.
\item \label{cond:gs_difference_t_bound} For any distinct $ij\in E(K_n)$, $\|\bt_i^*-\bt_j^*\|\leq c_0\mu$, where
\begin{equation}\label{defmu}
\mu=\frac{1}{|E(K_n)|}\sum_{ij\in E(K_n)}\|\bt_i^*-\bt_j^*\|.
\end{equation}
\item \label{cond:gs_max_degree}The maximal degree of $\Egl^c$ is $\epsilon_0 n$.
\item \label{cond:gs_well_dist_G}$T$ is $c_1$-well-distributed along $G$ and along $K_n$.
\item \label{cond:gs_no_collinear} For any distinct $i$, $j$ ,$k \in [n]$, $\bt_i^*$, $\bt_j^*$ and $\bt_k^* \in V$ are not collinear.
\end{enumerate}
\end{definition}
At last, we claim that under the HLV model the good-shape condition with certain restriction on its parameters implies exact recovery. The proof verifies that the good-long-dominance condition holds and then applies Theorem \ref{thm:deterministic0}.
\begin{customthm}{3}\label{thm:deterministic}
If $\{\bt_i^*\}_{i=1}^n$, $E = \Eg \cup \Eb \subseteq[n]\times[n]$ and $\{\bga_{ij}\}_{ij\in E}$ satisfy the good-shape condition with respect to the parameters
$p$, $\beta$, $\epsilon_0$, $\epsilon_1$, $c_1$, $c_0$ and if
\begin{equation}\label{eq:thm3}\epsilon_0\leq \min\left\{\frac{\beta c_1p}{2^{22}c_0^3},\frac{\beta c_1^2 p}{2^{20}c_0}, \frac{c_1p^2}{16}\right\}\,\,\text{ and }\,\,\,\epsilon_1\leq \min\left(\frac{1}{144c_0},\frac{1}{96}\right),\end{equation}
then the solution $\{\hat\bt_i\}_{i=1}^n$ of \eqref{eq:LUD1} has the form $\hat\bt_i=c^*\bt_i^*+\bt_s$ for $i\in [n]$, where $c^*$ and $\bt_s$ solve \eqref{eq:oracle}.
\end{customthm}

\subsection{Conclusion of Theorem \ref{thm:main}}\label{sec:ver}
We verify that under the HLV model the good-shape condition holds with parameters satisfying~\eqref{eq:thm3} and with high probability. Combining this observation with Theorem~\ref{thm:deterministic} results in Theorem~\ref{thm:main}.

We assume the conditions of Theorem~\ref{thm:main} and set the following parameters
 \[\beta=\frac{p}{2^{18} \log n} , \ c_1=\frac{c}{\sqrt{\log n}}  , \ \epsilon_1=\frac{p}{192 c_0} \text{ and } c_0=64\sqrt{\log n},\]
where $c$ is a constant used in Lemma 3.10 of \cite{HandLV15}.
The second inequality of \eqref{eq:thm3} is clearly satisfied with these parameters. We note that establishing the first inequality
of \eqref{eq:thm3} requires establishing the inequality $\epsilon_0 \leq c' p^2/\log^3 n$, where $c'$ linearly depends on $c$,  that is,  $\epsilon_0 = O(p^2/\log^3 n)$.
The following theorem, which is proved in Section \ref{sec:proof3}, establishes this under the assumptions of Theorem \ref{thm:main}.
\begin{customthm}{4}\label{thm:corruptions}
If the camera locations $\{\bt_i^*\}_{i=1}^n$ and pairwise directions $\{\bga_{ij}\}_{ij\in E}$ are generated by the HLV model with $p=\Omega(\sqrt[3]{\log n/n})$ and $\epsilon_b=O(p^{7/3}/\log^{9/2} n)$, then
\begin{equation}\label{eq:thm4eps}\epsilon_0 = O\left(p^2/\log^3 n\right)  \ \text{ w.p. } 1-O(n^{-5}).
\end{equation}
\end{customthm}

At last, we note that Lemma 3.7 of \cite{HandLV15} and the assumption of Theorem~\ref{thm:main} that $p=\Omega(\sqrt[3]{\log n/n})$\footnote{Recall that for $a$, $b \in \mathbb{R}$, the notation $a= \Omega(b)$ is equivalent with $b=O(a)$.} imply property \ref{cond:gs_p_typical} of Definition \ref{def:goodshape} with probability larger than $1-O(n^{-5})$. Lemma 3.10 of \cite{HandLV15} and the assumption of Theorem~\ref{thm:main} that $p=\Omega(\sqrt[3]{\log n/n})$ imply properties \ref{cond:gs_large_angles}, \ref{cond:gs_difference_t_bound} and \ref{cond:gs_well_dist_G} of Definition \ref{def:goodshape} with probability $1-O(n^{-5})$ and with the above choice of parameters. Property \ref{cond:gs_max_degree} of Definition \ref{def:goodshape} is just the definition of $\epsilon_0$, where the size of $\epsilon_0$ was estimated in Theorem~\ref{thm:corruptions}.
Furthermore, property \ref{cond:gs_no_collinear} of Definition \ref{def:goodshape} holds almost surely since the vertices are generated by i.i.d.~Gaussian distributions.

We have shown that all properties of the good-shape condition  and \eqref{eq:thm3} hold with probability $1-O(n^{-5})$, which can be written as $1-n^{-4}$ for sufficiently large $n$. This concludes the proof of Theorem~\ref{thm:main}.

We remark that the bound on $\epsilon_b$ in Theorem \ref{thm:main} is chosen so that \eqref{eq:thm4eps} and the first inequality of \eqref{eq:thm3} hold. Note that the lower bound on $p$ in Theorem \ref{thm:main} is sufficient for Theorem \ref{thm:corruptions}. As mentioned earlier, this lower bound can be modified to be of order $n^{\delta-1/2}\log^{1/2-\delta} n$ for any positive $\delta$ sufficiently small.

\subsection{Novelties of This Paper}\label{sec:nov}

This work uses ideas and techniques of \cite{HandLV15}, but considers LUD instead of ShapeFit and guarantees a stronger rate of corruption. Here
we highlight the main technical differences between the two works and emphasize the novel arguments for handling these differences in the current work.

\textbf{Reformulation:} 
The objective function of ShapeFit depends only on $\{\bt_i\}_{i=1}^n$, while the objective function of LUD has the additional variables $\{\alpha_{ij}\}_{ij\in E}$, which introduce more degrees of freedom.
To handle this issue, we reformulated the LUD problem in \eqref{eq:LUD1} as an equivalent convex optimization problem with objective function depending only on $\{\bt_i\}_{i=1}^n$. We also needed to introduce the oracle problem \eqref{eq:oracle} that provided the scale and shift of LUD with respect to the ground truth.
 Furthermore, we needed to guarantee uniqueness of the oracle scale, $c^*$, with overwhelming probability. The latter guarantee is restricted to the corrupted case and thus required us to guarantee parallel rigidity with overwhelming probability in the uncorrupted case.

\textbf{Adaptation to the new formulation:} The reformulated objective function for LUD is different than that of ShapeFit only in the case where $\langle \bga_{ij}\,, \hat \bt_i-\hat \bt_j\rangle\leq 1$.
We note that for $ij \in \Egl$, $\langle \bga_{ij}\,, \hat \bt_i-\hat \bt_j\rangle>1$. 
Therefore, for $ij \in \Egl$ the objective functions of ShapeFit and LUD coincide. Our analysis thus tries to follow that of \cite{HandLV15}, while replacing
$\Eg$ and $\Eb$ in \cite{HandLV15} with $\Egl$ and $\Egl^c$ respectively. Some modifications in the analysis of \cite{HandLV15} are needed, in particular, the two mentioned below.

\textbf{More faithful constraint on perturbation:}
Both works introduce constraints on the perturbed solutions $\{c^*\bt_i^* + \bt_s +\beps_i\}_{i=1}^n$. Even though $c^*$ is not defined in  \cite{HandLV15}, it can be defined as the constant satisfying $\sum_{ij\in E}\langle c^*\bt_i^*-c^*\bt_j^*\,,\bga_{ij}\rangle=1$, where the ground truth $\{\bt_i^*\}_{i=1}^n$ is denoted by $\{\bt_i^0\}_{i=1}^n$ in \cite{HandLV15}.
Hand, Lee and Voroninski \cite{HandLV15} require that
\begin{equation}\label{eq:handconst0}
\sum_{ij\in E} \langle \beps_i-\beps_j\,, \bga_{ij}\rangle =0
\end{equation}
so that any perturbed solution $\{\tilde\bt_i\}_{i=1}^n$, where $\tilde\bt_i=c^*\bt_i^*  + \bt_s +\beps_i$ for all $i\in [n]$, satisfy
\begin{equation*}\sum_{ij\in E}\langle \tilde\bt_i-\tilde\bt_j\,,\bga_{ij}\rangle=1. \end{equation*}
The perturbation constraint of our work appears in the formulation of the good-long-dominance condition. That is, the perturbation vectors $\{\beps_i\}_{i=1}^n$ need to satisfy $\sum_{i=1}^n \langle \beps_i\,, \bt_i^*\rangle =0$ and $\sum_{i=1}^n\beps_i=0$. This requirement implies that
\begin{equation}\label{eq:epsconstr}
\sum_{ij\in E(K_n)}\langle\beps_i-\beps_j,\bt_{ij}^*\rangle=0.
\end{equation}
We note that the perturbation constraint in \eqref{eq:epsconstr} replaces $\bga_{ij}$ and $E$ in \eqref{eq:handconst0} with $\bt_{ij}^*=\bt_i^*-\bt_j^*$ and $E(K_n)$ respectively. Any perturbed solution $\{\tilde\bt_i\}_{i=1}^n$ thus needs to satisfy
\begin{equation}\sum_{ij\in E(K_n)}\langle \tilde\bt_i-\tilde\bt_j\,,\bt_{ij}^*\rangle=\sum_{ij\in E(K_n)}\langle c^*\bt_i^*-c^*\bt_j^*\,,\bt_{ij}^*\rangle=c^*\sum_{ij\in E(K_n)}\|\bt_{ij}^*\|^2. \end{equation}
We believe that our perturbation constraint is more faithful to the underlying structure of the problem. First of all, it uses the correct directions $\bt_{ij}^*$ instead of the corrupted ones $\bga_{ij}$. More importantly, it uses $\bt_{ij}^*$ for any pair of locations, even if they are not connected by an edge. The latter property results in improved estimates in comparison to \cite{HandLV15}. For example, our lower bound in \eqref{eq:lowereta3} is tighter than the one in \cite[page 38]{HandLV15}, which is multiplied by $2p^2$ and suffers when $p\ll 1$.

\textbf{Effective way of controlling ${\boldsymbol \epsilon_0}$:} A deterministic upper bound on $\epsilon_b$ was obtained in
page 31 of \cite{HandLV15}, where $\epsilon_b$ is denoted in \cite{HandLV15} by $\epsilon_0$.
A direct analogous bound on the maximal degree of $\Egl^c$, $\epsilon_0$, depends on the unknown scale $c^*$ and is thus not appealing. The proof of Theorem~\ref{thm:corruptions} shows that with high probability $1/c^*$ concentrates around a function of $\epsilon_b$, $n$ and $p$ and consequently $\epsilon_0$ can also be controlled with high probability by a function of $\epsilon_b$, $n$ and $p$, as stated in Theorem~\ref{thm:corruptions}. The proof of this theorem is delicate and does not follow ideas of \cite{HandLV15}.

\section{Proof of Theorem~\ref{thm:deterministic0}}\label{sec:proof1}
We assume WLOG that $\bt_s=\b0$, or equivalently $\sum_{i=1}^n\bt_i^*=\b0$. 
Indeed, the statement of Theorem~\ref{thm:deterministic0}, in particular, the good-long-dominance condition, is independent of any shift of the locations $\{\bt_i^*\}_{i=1}^n$.

Since the objective function in \eqref{eq:LUD1} is convex, in order to prove that $\{c^*\bt_i^*\}_{i=1}^n$ solves~\eqref{eq:LUD1},
it is sufficient to prove that for any sufficiently small perturbations $\{\beps_i\}_{i=1}^n\in\reals^3$ such that $\sum_{i=1}^n\beps_i=\b0$,
\begin{equation}\label{eq:intermediate}
\sum_{ij\in E} f_{ij}(c^*\bt_i^*+\beps_i\,,c^*\bt_j^*+\beps_j)\geq \sum_{ij\in E} f_{ij}(c^*\bt_i^*\,,c^*\bt_j^*).
\end{equation}

We note that there exists $\kappa\in \mathbb{R}$ such that for any $i\in [n]$, $\beps_i$ can be decomposed as
$\beps_i=\beps_i^{\myparallel}+\beps_i^{\perp}$, where ${\beps_i^{\myparallel}}=\kappa\bt_i^*$ and $\sum_{i=1}^n\langle\beps_i^{\perp}\,, \bt_i^*\rangle =0$.
To clarify this, we stack the elements of $\{\beps_i\}_{i=1}^n$, $\{\beps_i^{\myparallel}\}_{i=1}^n$, $\{\beps_i^{\perp}\}_{i=1}^n$, $\{\bt_i^*\}_{i=1}^n$ as columns of the respective matrices $\bSigma$, $\bSigma^{\myparallel}$, $\bSigma^\perp$ and $\bT^*$ so that $\bSigma^{\myparallel}=\kappa \bT^*$, $\bSigma=\bSigma^{\myparallel}+\bSigma^\perp$ and $\langle \bSigma^{\myparallel}\,, \bSigma^\perp\rangle=\trace(\bSigma^{\myparallel \bT} \bSigma^\perp)=0$.
Furthermore, the assumption $\bt_s=\b0$ implies that $\sum_{i=1}^n \beps_i^{\perp} = \sum_{i=1}^n \beps_i = \b0$. Therefore, the perturbations $\{\beps_i^{\perp}\}_{i=1}^n$ satisfy the required assumptions on the perturbations used in the good-long-dominance condition.

Letting $c'=c^*+\kappa$, the relation $\beps_i= \kappa\bt_i^* +\beps_i^{\perp}$ implies that
\begin{equation}
\label{eq:c*_c'}
c^*\bt_i^*+\beps_i=c'\bt_i^*+\beps_i^{\perp} \  \ \text{ for all }  i\in [n].
\end{equation}
Since $\{\beps_i\}_{i=1}^n$ have sufficiently small norms, we may assume that $c'$ is sufficiently close to $c^*$.

Next, we obtain useful estimates in two complementary cases.\\
\noindent \textbf{Case A: ${\boldsymbol{ ij \in \Egl}}$}.
In this case, $\bga_{ij}=(\bt_i^*-\bt_j^*)/\|\bt_i^*-\bt_j^*\|=\bga_{ij}^*$ and  $\| P_{\bga_{ij}}(c^*(\bt_i^*-\bt_j^*))\|>1$. Combining the latter inequality, the fact that the perturbations are arbitrarily small and the proximity of $c'$ to $c^*$ result in $\| P_{\bga_{ij}}(c'(\bt_i^*-\bt_j^*)+\beps_i^{\perp}-\beps_j^{\perp})\|>1$. Applying \eqref{eq:c*_c'}, then the latter inequality and \eqref{eq:def_f}, and at last the assumption $ij \in \Egl$ concludes that
\begin{multline*}
f_{ij}(c^*\bt_i^*+\beps_i\,,c^*\bt_j^*+\beps_j)=f_{ij}(c'\bt_i^*+\beps_i^{\perp}\,,c'\bt_j^*+\beps_j^{\perp})
=\| P_{\bga_{ij}^{\perp}}(c'(\bt_i^*-\bt_j^*)+\beps_i^{\perp}-\beps_j^{\perp})\|
=\| P_{\bga_{ij}^{\perp}}(\beps_i^{\perp}-\beps_j^{\perp})\|.
\end{multline*}
This equation and the observation $f_{ij}(c'\bt_i^*\,,c'\bt_j^*)=0$ imply the inequality:
\begin{equation}\label{eq:caseA}
\sum_{ij\in \Egl}\left(f_{ij}(c^*\bt_i^*+\beps_i\,,c^*\bt_j^*+\beps_j)-f_{ij}(c'\bt_i^*\,,c'\bt_j^*)\right)=\sum_{ij\in \Egl}\| P_{\bga_{ij}^{\perp}}(\beps_i^\perp-\beps_j^\perp)\|.
\end{equation}

\noindent \textbf{Case B: ${\boldsymbol{ ij\in \Egl^c}}$}.
Following the demonstration in Figures \ref{fig:1} and \ref{fig:2}, we note that $f_{ij}(\bt_i,\bt_j)$ is the distance between the following two convex sets:  $\{\alpha \bga_{ij}:\alpha\geq 1\}$ and
the singleton $\{\bt_i-\bt_j\}$. Application of \eqref{eq:c*_c'} and then the triangle inequality for a distance between convex sets of $\mathbb{R}^3$ results in
\begin{multline}
|f_{ij}(c^*\bt_i^*+\beps_i\,,c^*\bt_j^*+\beps_j)-f_{ij}(c'\bt_i^*\,,c'\bt_j^*)|
=|f_{ij}(c'\bt_i^*+\beps_i^{\perp}\,,c'\bt_j^*+\beps_j^{\perp})-f_{ij}(c'\bt_i^*\,,c'\bt_j^*)|\leq \|\beps_i^\perp-\beps_j^\perp\|.\label{eq:caseB}
\end{multline}

At last, we combine the above estimates with the good-long-dominance condition to verify \eqref{eq:intermediate}. We first apply \eqref{eq:caseA}, then the good-long-dominance condition of \eqref{eq:deterministic0} with $\{\beps_i^{\perp}\}_{i=1}^n$ that satisfy its necessary requirements, and at last \eqref{eq:caseB}, and consequently conclude that
\begin{align*}
\sum_{ij\in \Egl}\left(f_{ij}(c^*\bt_i^*+\beps_i\,,c^*\bt_j^*+\beps_j)-f_{ij}(c'\bt_i^*\,,c'\bt_j^*)\right)
\geq 
\sum_{ij\in \Egl^c} \left(f_{ij}(c'\bt_i^*\,,c'\bt_j^*)-f_{ij}(c^*\bt_i^*+\beps_i\,,c^*\bt_j^*+\beps_j)\right).
\end{align*}
By rearranging terms, this equation becomes
\begin{align*}
\sum_{ij\in E}f_{ij}(c^*\bt_i^*+\beps_i\,,c^*\bt_j^*+\beps_j)
\geq &\sum_{ij\in E} f_{ij}(c'\bt_i^*\,,c'\bt_j^*).
\end{align*}
By the definition of $c^*$ in~\eqref{eq:oracle} and the assumption $\bt_s=\b0$, this equation implies \eqref{eq:intermediate} and thus concludes the proof.
\section{Proof of Theorem~\ref{thm:deterministic}}\label{sec:proof2}
We show that under the assumptions of Theorem \ref{thm:deterministic}, the good-shape condition implies the good-long-dominance condition and consequently Theorem \ref{thm:deterministic} follows from Theorem \ref{thm:deterministic0}. Section~\ref{prelim}  reviews notation and auxiliary lemmas, which were borrowed from \cite{HandLV15}. Section \ref{sec:details_proof_3}
presents the details of the proof.

While the outline of the proof in this section resembles the outline of the proof of Theorem 3.4 of~\cite{HandLV15}, there are some nontrivial modifications. A main difference between the proofs appears in the perturbation constraints stated earlier in \eqref{eq:handconst0} and \eqref{eq:epsconstr}. 

\subsection{Preliminaries}\label{prelim}
We first review some notation that we mainly borrowed from \cite{HandLV15}. We denote $\bt_{ij}^*=:\bt_i^*-\bt_j^*$ and for $\{\beps_i\}_{i=1}^n\subseteq\mathbb{R}^3$, we define $\eta_{ij}=\|P_{\bga_{ij}^{*\perp}}(\beps_i-\beps_j)\|$ and $\delta_{ij}\|\bt_{ij}^*\|=\langle \beps_i-\beps_j, \bga_{ij}^*\rangle$. We note that $\beps_i-\beps_j$ is the motion of relative location $\bt_i^*-\bt_j^*$ after perturbing $\bt^*_1, \ldots, \bt^*_n$ respectively by $\beps^*_1, \ldots, \beps^*_n$. Thus for edge $ij$, $\eta_{ij}$
is the component of the motion that is orthogonal to $\bt_i^*-\bt_j^*$ and is referred to as rotational motion. Similarly, for edge $ij$, $\delta_{ij}\|\bt_{ij}^*\|$
is the component of the motion that is parallel to $\bt_i^*-\bt_j^*$ and is referred to as parallel motion.  The function $\eta:$ $E(K_n)\times E(K_n)\to \mathbb{R}$ of \cite{HandLV15} is defined as
\begin{equation}\label{eq:eta}
\eta(ij,kl)=\sum_{\latop{m,n \in\{i,j,k,l\}}{m< n}}\eta_{mn}.
\end{equation}
That is, if $ij$ and $kl$ do not have common elements, then $\eta(ij,kl)=\eta_{ij}+\eta_{kl}+\eta_{ik}+\eta_{il}+\eta_{jk}+\eta_{jl}$. If they have one common element, e.g., $i=k$, then $\eta(ij,kl)=\eta_{ij}+\eta_{il}+\eta_{jl}$.
We modify the definition of $\Eg'$ in~\cite{HandLV15} and define $E'(K_n)$ as follows:
\begin{equation}E'(K_n)=\{ij\in E(K_n): \|\bt_{ij}^*\|\geq \frac{1}{2}\mu\},\end{equation} where $\mu$ was defined in equation \eqref{defmu}.
Let $B(ij)$ denote the set of all $kl\in E(K_n)$ for which there exist distinct $a,b,c\in\{i,j,k,l\}$ satisfying $\{a,b\}\neq \{i,j\}$ and $\sqrt{1-\langle\bga_ac^*\,,\bga_{bc}^*\rangle}<\beta$.

The following lemmas are from \cite{HandLV15}. We remark that Lemma~\ref{thm:lemma5} was formulated in \cite{HandLV15} for $E' = \Eg$ as a matter of convenience, however, its formulation below still hold.
\begin{customlemma}{1}[Lemma 2.6 of \cite{HandLV15} with $\alpha = 1$]\label{thm:lemma3}
Let $K_4$ be the complete graph of 4 vertices with 4 distinct vertex locations $\{\bt_i^*\}_{i=1}^4 \subset \mathbb{R}^3$, and let $\{\beps_i\}_{i=1}^4 \subset \mathbb{R}^3$ be perturbation vectors. Then
\begin{equation}
\eta(12,34)\geq \frac{\beta_0}{4}\|\bt_{12}^*\||\delta_{12}-\delta_{34}|, \ \text{ where } \ \beta_0=\min\limits_{\latop{\{i,j,k\}\in[4]}{\{j,k\}\neq \{1,2\}}} \sqrt{1-\langle \bga_{ij}^*\,,\bga_{ik}^*\rangle}.
\end{equation}
\end{customlemma}
\begin{customlemma}{2}[Lemmas 2.8 and 2.9 of \cite{HandLV15}]\label{thm:lemma5}
Let $G([n],E)$ be $p$-typical and $c_1$-well-distributed graph with $n$ vertices for $0<p$, $c_1 \leq 1$ and let $E'$ be a subset of $E$, where the maximal degree of its complement, $E'^c$, is bounded by $\epsilon' n$. If $\epsilon' \leq c_1 p^2/8$, then
\begin{equation}
\sum_{ij\in E'}\eta_{ij}\geq \frac{c_1 p^2}{8\epsilon'}\sum_{ij\in E'^c}\eta_{ij} \text{ and } \sum_{ij\in E'}\eta_{ij}\geq \frac{c_1 p}{16}\sum_{ij\in E(K_n)}\eta_{ij}.
\end{equation}
\end{customlemma}
Since $K_n$ is $1$-typical, the next corollary follows from the first inequality of Lemma \ref{thm:lemma5}.
\begin{customcorollary}{1}\label{thm:coro1}
Let $K_n$ be $c_1$-well-distributed and let $E'$ be a subset of $E(K_n)$, where the maximal degree of its complement, $E'^c$, is bounded by $\epsilon' n$.
If $\epsilon' \leq c_1/8$, then
\begin{equation}
\sum_{ij\in E'}\eta_{ij}\geq \frac{c_1}{8\epsilon'}\sum_{ij\in E'^c}\eta_{ij}.
\end{equation}
\end{customcorollary}
\begin{customlemma}{3}[Lemma 3.6 of \cite{HandLV15}]\label{thm:lemma14}
For any $ij\in E(K_n)$,
\begin{equation}
|B(ij)|\leq 6\epsilon_1n^2,
\end{equation}
where $\epsilon_1$ is the constant specified in property~\ref{cond:gs_large_angles} of Definition~\ref{def:goodshape}.
\end{customlemma}

\subsection{Details of Proof} \label{sec:details_proof_3}
 In order to verify the good-long-dominance condition of
\eqref{eq:deterministic0}, it is sufficient to prove that the total rotational motion on  $\Egl$ is greater than or equal to    two times the total parallel motion on  $\Egl^c$. That is,
 \begin{equation}\label{eq:sufficient}\sum_{ij\in \Egl}\eta_{ij}\geq  2\sum_{ij\in \Egl^c}|\delta_{ij}|\|\bt_{ij}^*\|.\end{equation}
Indeed, since $\epsilon_0\leq c_1p^2/16$ we can apply the first inequality of Lemma \ref{thm:lemma5} and obtain that $$\sum_{ij\in \Egl}\eta_{ij} \geq 2\sum_{ij\in \Egl^c}\eta_{ij}.$$ The combination of the latter inequality with \eqref{eq:sufficient} and the triangle inequality $\|\beps_i-\beps_j\|\leq |\delta_{ij}| \|\bt_{ij}^*\|+\eta_{ij}$ yields \eqref{eq:deterministic0}.

Following \cite{HandLV15}, we prove \eqref{eq:sufficient} by considering three complementary cases, which depend on the relative averaged parallel motion on $\Egl^c$, that is, $$\bar{\delta}=\sum_{ij\in \Egl^c}|\delta_{ij}|\|\bt_{ij}^*\|/\sum_{ij\in \Egl^c}\|\bt_{ij}^*\|.$$
These three cases can be simplistically categorized according to zero, large and small non-zero parallel motions on $\Egl^c$.

\textbf{Case 1:  $\boldsymbol{\bar{\delta}=0}$ or $\boldsymbol{\Egl^c=\emptyset}$.} \
Since either $\Egl^c=\emptyset$ or $\delta_{ij}=0$ for all $ij\in \Egl^c$,  the RHS of \eqref{eq:sufficient} is 0.

\textbf{Case 2: $\boldsymbol{\bar{\delta}\neq 0}$, $\boldsymbol{\Egl^c\neq\emptyset}$ and $\boldsymbol{\sum_{ij\in E'(K_n)}|\delta_{ij}|<\bar{\delta}|E'(K_n)|/8}$}. \
First, we obtain a lower bound on $|E'(K_n)|/|E(K_n)|$. The definition of $E'(K_n)$ and then the definition of $\mu$ in \eqref{defmu} result in
\[
\sum_{ij\in E(K_n)\setminus E'(K_n)}\|\bt_{ij}^*\|< \frac{1}{2}\mu |E(K_n)|= \frac{1}{2}\sum_{ij\in E(K_n)}\|\bt_{ij}^*\|.
\]
Consequently,
\begin{equation}\label{eq:Ekn}
\sum_{ij\in E'(K_n)}\|\bt_{ij}^*\|\geq \frac{1}{2}\sum_{ij\in E(K_n)}\|\bt_{ij}^*\|=\frac12 \mu |E(K_n)|.
\end{equation}
Using assumption~\ref{cond:gs_difference_t_bound} of the good-shape condition (Definition~\ref{def:goodshape}) and then \eqref{eq:Ekn}, we obtain that
\[
c_0\mu |E'(K_n)|\geq \sum_{ij\in E'(K_n)}\|\bt_{ij}^*\| \geq \frac{1}{2}\mu|E(K_n)|
\]
and consequently
 \begin{equation}\label{eq:est_E'}|E'(K_n)|\geq \frac{1}{2c_0}|E(K_n)|.\end{equation}

We change the definition of $L_b$ in \cite{HandLV15} to $L=\{ij\in \Egl^c: |\delta_{ij}|\geq \frac{1}{2}\bar{\delta}\}$ and derive the following inequality, which is analogous to (14) of \cite{HandLV15}:
\begin{equation}\label{eq:L2}
\sum_{ij\in L}|\delta_{ij}|\|\bt_{ij}^*\|=
\sum_{ij\in \Egl^c}|\delta_{ij}|\|\bt_{ij}^*\|
-\sum_{ij\in \Egl^c \setminus L}|\delta_{ij}|\|\bt_{ij}^*\|
\geq \frac{1}{2}\sum_{ij\in \Egl^c}|\delta_{ij}|\|\bt_{ij}^*\|.
\end{equation}
We modify the definition of $F_g$ in~\cite{HandLV15} to $F'(K_n)=\{ij\in E'(K_n): |\delta_{ij}|<\frac{1}{4}\bar{\delta}\}$ and following \cite{HandLV15}, while using the last assumption of this case (case 2), we obtain that
\[
\frac{1}{8}\bar{\delta}|E'(K_n)|>\sum_{ij\in E'(K_n)}|\delta_{ij}|\geq \sum_{ij\in E'(K_n)\setminus F'(K_n)}|\delta_{ij}|\geq\frac{1}{4}\bar{\delta}|E'(K_n)\setminus F'(K_n)|.
\]
We thus conclude that $|F'(K_n)|>\frac{1}{2}|E'(K_n)|$. Combining this inequality with \eqref{eq:est_E'} we conclude that for $n\geq 3$,
\begin{equation}\label{eq:Fkn}|F'(K_n)| > \frac{1}{4c_0}|E(K_n)|=\frac{n(n-1)}{8c_0}\geq \frac{n^2}{12c_0}.\end{equation}

By Lemma \ref{thm:lemma14}, $|B(ij)|\leq 6\epsilon_1 n^2$ for all $ij\in E(K_n)$. Combining this with \eqref{eq:Fkn}, we obtain that for $\epsilon_1\leq \frac{1}{144c_0}$,
\begin{equation}\label{eq:F1Bij}|F'(K_n)\setminus B(ij)| > \frac{n^2}{12c_0}-6\epsilon_1 n^2\geq \frac{n^2}{24c_0}.\end{equation}

The rest of the proof uses the above inequalities to obtain a lower bound on the LHS of \eqref{eq:sufficient} and a similar upper bound on the RHS of \eqref{eq:sufficient}.
To get the lower bound, we first note that the second inequality of Lemma \ref{thm:lemma5} implies that
\begin{equation}\label{eq:lemma6}
\sum_{ij\in \Egl}\eta_{ij}\geq \frac{c_1p}{16}\sum_{ij\in E(K_n)}\eta_{ij}.
\end{equation}
We thus need to find a lower bound for the RHS of \eqref{eq:lemma6}.

We next establish the inequality
\begin{equation}\label{eq:counting}
\sum_{ij\in \Egl^c}\sum_{\latop{kl\in E(K_n)}{kl\neq ij}}\eta(ij,kl)\leq \sum_{ij\in \Egl^c}3n^2\eta_{ij}+\sum_{ij\in E(K_n)}18\epsilon_0n^2\eta_{ij}
\end{equation}
by following 
a combinatorial argument of \cite{HandLV15} (see case 1 in the proof of Theorem 3.4 in \cite{HandLV15}). There are two differences in our cases. First, we replace $\Eb$ and $\Eg$, which are used in \cite{HandLV15}, with  $\Egl^c$ and $E(K_n)$. Second, the sets $\Egl^c$ and $E(K_n)$ have nonempty intersection, unlike $\Eb$ and $\Eg$.
The argument is that any fixed $ij$ in the first sum in the LHS of \eqref{eq:counting} appears in at most $n\choose 2$ $K_4$'s, where the other two vertices are chosen from the second sum, and at most $n$ $K_3$'s, where another vertex and either $i$ or $j$ are from the second sum. Therefore, when fixing $ij$ in the first sum, $\eta_{ij}$ can appear at most $6\cdot$ $n \choose 2$ $+ 3n = 3n^2$ times.
On the other hand, any fixed $kl$ in the second sum belongs to either $K_4$ or $K_3$ containing $ij$ in the first sum. By applying assumption \ref{cond:gs_max_degree} of the good-shape condition, $kl$ belongs to at most $ 2\epsilon_0 n(n-3)$ $K_4$'s, where $ij$ is incident to $kl$, $\epsilon_0 n^2$ $K_4$'s, where $ij$ is not incident to $kl$, and $2\epsilon_0 n$ $K_3$'s.
Therefore, when fixing $kl$ in the second sum, $\eta_{kl}$ can appear at most $6 \cdot 2 \epsilon_0 n (n-3) + 6 \epsilon_0 n^2 + 3 \cdot 2 \epsilon_0 n \leq 18 \epsilon_0 n^2$ times.

We recall that $\epsilon_0\leq c_1p^2/8\leq c_1/8$ and thus Corollary \ref{thm:coro1} implies that
\begin{equation*}
\sum_{ij\in E(K_n)}\eta_{ij}\geq  \sum_{ij\in E(K_n) \setminus \Egl^c}\eta_{ij}\geq \frac{c_1}{8\epsilon_0}\sum_{ij\in  \Egl^c}\eta_{ij}.
\end{equation*}
The above two inequalities yield
\begin{align}\label{eq:upper}
&\sum_{ij\in \Egl^c}\sum_{\latop{kl\in E(K_n)}{kl\neq ij}}\eta(ij,kl)\leq \frac{42\epsilon_0}{c_1}n^2\sum_{ij\in E(K_n)}\eta_{ij}.
\end{align}
The combination of \eqref{eq:lemma6} and \eqref{eq:upper} results in the following lower bound on the LHS of \eqref{eq:sufficient}
\begin{equation}
\label{eq:lower_LHS}
\sum_{ij\in \Egl}\eta_{ij} \geq \frac{c_1p}{16}\sum_{ij\in E(K_n)}\eta_{ij}\geq \frac{c_1^2 p}{3 \cdot 2^8  \epsilon_0 n^2}\sum_{ij\in \Egl^c}
\sum_{\latop{kl\in E(K_n)}{kl\neq ij}}\eta(ij,kl).
\end{equation}

In order to upper bound the RHS of \eqref{eq:sufficient} we first apply Lemma \ref{thm:lemma3}, which implies that for $ij\in L$  and $kl\in F'(K_n)\setminus B(ij)$
\[
\eta(ij,kl)\geq \frac{\beta}{4}|\delta_{kl}-\delta_{ij}|\|\bt_{ij}^*\|.
\]
For $ij\in L$, $|\delta_{ij}|>\frac12\bar\delta$ and for $kl\in F'(K_n)$, $|\delta_{kl}|<\frac14\bar\delta$. Consequently, for $ij\in L$ and $kl\in F'(K_n)\setminus B(ij)$, $|\delta_{kl}| < |\delta_{ij}|/2$ and
\begin{equation}\label{eq:etalower}
\eta(ij,kl)\geq \frac{\beta}{4}\big||\delta_{kl}|-|\delta_{ij}|\big|\|\bt_{ij}^*\|\geq\frac{\beta}{8}|\delta_{ij}|\|\bt_{ij}^*\|.
\end{equation} Applying first the inclusions $L\subseteq \Egl^c$ and $F'(K_n)\subseteq E(K_n)$, then \eqref{eq:etalower}, next \eqref{eq:F1Bij} and at last \eqref{eq:L2}, we obtain that
\begin{align*}\nonumber
&\sum_{ij\in \Egl^c}\sum_{\latop{kl\in E(K_n)}{kl\neq ij}}\eta(ij,kl)\geq
\sum_{ij\in L}\hspace{0.2cm}\sum_{kl\in F'(K_n)\setminus B(ij)}\eta(ij,kl) \\ &\geq
\sum_{ij\in L}|F'(K_n)\setminus B(ij)|\cdot \frac{\beta}{8}|\delta_{ij}|\|\bt_{ij}^*\|> \frac{\beta}{8}\cdot\frac{n^2}{24c_0}\sum_{ij\in L}|\delta_{ij}|\|\bt_{ij}^*\|\nonumber
\geq \frac{\beta}{16}\cdot\frac{n^2}{24c_0}\sum_{ij\in \Egl^c}|\delta_{ij}|\|\bt_{ij}^*\|.
\end{align*}
This equation implies the following upper bound for the RHS of \eqref{eq:sufficient}:
\begin{equation}
2 \sum_{ij\in \Egl^c}|\delta_{ij}|\|\bt_{ij}^*\| <
\frac{3 \cdot 2^8 c_0}{\beta n^2} \sum_{ij\in \Egl^c}\sum_{\latop{kl\in E(K_n)}{kl\neq ij}}\eta(ij,kl).
\label{eq:upper_RHS}
\end{equation}

Note that \eqref{eq:thm3} implies that the RHS of~\eqref{eq:upper_RHS} is less than the RHS of~\eqref{eq:lower_LHS}. This observation concludes \eqref{eq:sufficient} and consequently  the proof of the current case.

\textbf{Case 3: $\boldsymbol{\bar{\delta}\neq 0}$, $\boldsymbol{\Egl^c\neq\emptyset}$ and $\boldsymbol{\sum_{ij\in E'(K_n)}|\delta_{ij}|\geq \bar{\delta}|E'(K_n)|/8}$.} \
Similarly to case 2, in order to prove \eqref{eq:sufficient}, we obtain a lower bound for the LHS of \eqref{eq:sufficient} and a similar upper bound  for the RHS of \eqref{eq:sufficient}.

Following \cite{HandLV15}, we define $E_+=\{ij\in E(K_n): \delta_{ij}\geq 0\}$ and $E_-=\{ij\in E(K_n): \delta_{ij}<0\}$.
Using this notation, we rewrite the perturbation constraint of \eqref{eq:epsconstr} 
as
\[
\sum_{ij\in E_+}\delta_{ij}\|\bt_{ij}^*\|^2+\sum_{ij\in E_-}\delta_{ij}\|\bt_{ij}^*\|^2=0
\]
and conclude that
\begin{equation}\label{Eplushalf}
\sum_{ij\in E_+}|\delta_{ij}|\|\bt_{ij}^*\|^2=\sum_{ij\in E_-}|\delta_{ij}|\|\bt_{ij}^*\|^2=\frac{1}{2}\sum_{ij\in E(K_n)}|\delta_{ij}|\|\bt_{ij}^*\|^2.
\end{equation}

Next, we upper bound the RHS of \eqref{eq:sufficient} by a constant times the term $\sum_{ij\in E_-}\sum_{kl\in E_+}\eta(ij,kl)$.
We first lower bound the latter term by following \cite{HandLV15} and applying Lemma \ref{thm:lemma3} as follows
\begin{align*}
\sum_{ij\in E_-}\sum_{kl\in E_+}\eta(ij,kl)
\geq \sum_{ij\in E_-}\sum_{kl\in E_+\setminus B(ij)}\frac{\beta}{4}|\delta_{ij}|\|\bt_{ij}^*\| 
\geq \frac{\beta}{4}(|E_+|-|B(ij)|)\sum_{ij\in E_-}|\delta_{ij}|\|\bt_{ij}^*\|.
\end{align*}
The successive application of property \ref{cond:gs_difference_t_bound} of the good-shape condition, \eqref{Eplushalf}, the inclusion $E'(K_n)\subseteq E(K_n)$, the definition of $E'(K_n)$ together with the assumption $\sum_{ij\in E'(K_n)}|\delta_{ij}|\geq \frac{1}{8}\bar{\delta}|E'(K_n)|$ and \eqref{eq:est_E'} results in
\begin{align}\nonumber
&\sum_{ij\in E_-}|\delta_{ij}|\|\bt_{ij}^*\|\geq \frac{1}{c_0\mu} \sum_{ij\in E_-}|\delta_{ij}|\|\bt_{ij}^*\|^2= \frac{1}{2c_0\mu}\sum_{ij\in E(K_n)}|\delta_{ij}|\|\bt_{ij}^*\|^2\\
&\geq
\frac{1}{2c_0\mu}\sum_{ij\in E'(K_n)}|\delta_{ij}|\|\bt_{ij}^*\|^2
\geq \frac{1}{2c_0\mu}\cdot \frac{1}{4}\mu^2\cdot \frac{1}{8}\bar\delta |E'(K_n)| \geq \frac{\mu\bar\delta n^2}{512c_0^2}.\label{eq:lowerEminus}
\end{align}

Assuming $|E_+|\geq |E(K_n)|/2$ and combining \eqref{eq:lowerEminus}, the fact that $|E(K_n)|=n(n-1)/2\geq n^2/4$ for $n\geq 2$, and the assumption $\epsilon_1\leq 1/96$, gives
\begin{equation*}
\frac{\beta}{4}(|E_+|-|B(ij)|)\sum_{ij\in E_-}|\delta_{ij}|\|\bt_{ij}^*\|\geq \frac{\beta\mu\bar{\delta}n^2}{2048c_0^2}\Big(\frac{1}{2}|E(K_n)|-6\epsilon_1 n^2\Big)\geq \frac{\beta\mu\bar{\delta}n^4}{2^{15}c_0^2}.
\end{equation*}
Consequently,
\begin{equation}\label{eq:lowereta3}
\sum_{ij\in E_-}\sum_{kl\in E_+}\eta(ij,kl)\geq \frac{\beta\mu\bar{\delta}n^4}{2^{15}c_0^2}.
\end{equation}
Assuming on the contrary that $|E_-|\geq |E(K_n)|/2$ and following the same arguments, while switching between $E_+$ and $E_-$, also yield
\eqref{eq:lowereta3}.

We conclude with the following upper bound on the RHS of \eqref{eq:sufficient} by first applying the definition of $\bar\delta$, then condition \ref{cond:gs_difference_t_bound} of Definition \ref{def:goodshape}, then condition \ref{cond:gs_max_degree} of Definition \ref{def:goodshape}, and at last \eqref{eq:lowereta3}:
\begin{align}
\sum_{ij\in \Egl^c}|\delta_{ij}|\|\bt_{ij}^*\| 
=\bar \delta\sum_{ij\in \Egl^c}\|\bt_{ij}^*\|\leq \bar{\delta}c_0\mu|\Egl^c|\leq \bar{\delta}c_0\mu\epsilon_0n^2
\leq \frac{2^{15}c_0^3\epsilon_0}{\beta n^2}\sum_{ij\in E_-}\sum_{kl\in E_+}\eta(ij,kl).\label{eq:uppereta3}
\end{align}

In order to obtain a lower bound on the LHS of \eqref{eq:sufficient}, we use the following result from \cite[page 38]{HandLV15}, which is obtained by counting the number of elements in the sum of $\eta$'s:
\begin{equation}\label{eq:etacounting3}
\sum_{ij\in E_-}\sum_{kl\in E_+}\eta(ij,kl)\leq 3n^2\sum_{ij\in E(K_n)}\eta_{ij}.
\end{equation}
We remark that although we modified the definition of $E_+$ and $E_-$, this result still holds. We conclude a lower bound on the LHS of \eqref{eq:sufficient} by applying the second inequality of Lemma \ref{thm:lemma5} and then \eqref{eq:etacounting3} as follows:
\begin{equation}\label{eq:lowerboundeta3}
\sum_{ij\in \Egl}\eta_{ij}\geq \frac{c_1p}{16}\sum_{ij\in E(K_n)}\eta_{ij}
\geq \frac{c_1p}{48n^2}\sum_{ij\in E_-}\sum_{kl\in E_+}\eta(ij,kl).
\end{equation}
The combination of \eqref{eq:uppereta3}, \eqref{eq:lowerboundeta3} and the assumption $\frac{\beta c_1p}{ 2^{21}c_0^3\epsilon_0}\geq 2$
verifies \eqref{eq:sufficient}.

\section{Proof of Theorem~\ref{thm:corruptions}}\label{sec:proof3}
It is sufficient to show that
\begin{equation}\label{eq:thm4wp}\epsilon_0 = O\left(\max\left\{p^2/\log^4 n\,, \,\,(p^{1/4}\log^{3/8}n)\cdot\epsilon_b^{3/4}\right\}\right)   \ \text{ w.p. } 1-O(n^{-5}).
\end{equation}
Indeed, combining \eqref{eq:thm4wp} with the assumption $\epsilon_b=O(p^{7/3}/\log^{9/2}n)$ of Theorem \ref{thm:corruptions}  implies that $\epsilon_0 = O(p^2/\log^3 n)$ and this concludes Theorem \ref{thm:corruptions}.

In the following we prove equation \eqref{eq:thm4wp}. Note that $\Egl^c\subseteq \Eb\cup E_s$, where $E_s=\{ij\in E:\|\bt_i^*-\bt_j^*\|<1/c^*\}$ is the set of short edges. Therefore, to conclude the theorem it is enough to estimate the maximal degree of $E_s$. Our estimate uses the following notation: $I$ denotes the indicator function, the neighborhood $N(\bt_i^*)$ of $\bt_i^*\in V$ includes all indices $j\in [n]$ such that $ij\in E$, and for $a$, $b \in \mathbb{R}$, $a \lesssim b$ if and only if $b= \Omega(a)$. We will prove that for any fixed $\bt_i^*\in V$
\begin{equation}\label{eq:lessmax}
\sum_{j\in N(\bt_i^*)} I\Big(\|\bt_i^*-\bt_j^*\|<\frac{1}{c^*}\Big) \lesssim \max\left\{\frac{np^2}{\log^4 n}\,,\, p^{\frac14}\epsilon_b^{\frac34}n\log^{\frac38}n\right\} \ \text{ w.p. } 1-O(n^{-6}).
\end{equation}
Taking a union bound yields
\[
\frac{\text{Maximal degree of $E_s$}}{n}\lesssim \max\left\{\frac{p^2}{\log^4 n}\,,\, p^{\frac14}\epsilon_b^{\frac34}\log^{\frac38}n\right\} \ \text{ w.p. } 1-O(n^{-5})
\] and this implies \eqref{eq:thm4wp} and thus concludes the proof of the theorem.

We derive \eqref{eq:lessmax} by using the following function of $c^*$, which is defined with respect to a Gaussian random variable $\bx\sim N(\b0,\bI)$ with pdf $\Phi$:
\begin{equation}\label{eq:gc}
g(c^*)=\Pr\Big(\Big\{\|\bx\|<\frac{1}{c^*}\Big\}\Big)=\int\limits_{B(\b0,\frac{1}{c^*})}\Phi(\bt)d\bt.
\end{equation}
We note that for fixed $\bt_i^*\in V$,
\begin{align}\label{eq:gcstar}
\Pr( \|\bt_i^*-\bt_j^*\|<1/c^*)
=\int_{B(\bt_i^*,\frac{1}{c^*})} \Phi(\bt)d\bt
\leq \int_{B(0,\frac{1}{c^*})} \Phi(\bt)d\bt=\Pr(\|\bt_j^*\|<1/c^*)=g(c^*).
\end{align}
 Furthermore,  $I(ij\in E \text{ and } \|\bt_i^*-\bt_j^*\|<1/c^*)$ is a Bernoulli random variable Bern$(\mu)$  with $\mu=p\Pr( \|\bt_i^*-\bt_j^*\|<1/c^*)\leq pg(c^*)$, where the last inequality follows from \eqref{eq:gcstar}. This observation and Chernoff bound can be used to conclude \eqref{eq:lessmax}. It is easily done in Section \ref{sec:thm4caseA} when $g(c^*)\lesssim 1/\sqrt n$, while only using the first term in the RHS of \eqref{eq:lessmax}. The other case, where $g(c^*)\gtrsim 1/\sqrt n$, is more complicated and verified in Section \ref{sec:thm4caseB} and uses the second term in the RHS of \eqref{eq:lessmax}.
\subsection{Proof for the case where $\boldsymbol{g(c^*)\lesssim 1/\sqrt n}$.}\label{sec:thm4caseA}

In order to verify \eqref{eq:lessmax},
 we use the following version of Chernoff bound~\cite{chernoff} for Bernoulli random variables: If $X_1, X_2,\cdots, X_n$ $\sim$ Bern$(\mu)$ i.i.d., then
\begin{equation}\label{eq:chernoff2}
\Pr\Big(\frac{1}{n}\sum_{i=1}^n X_i-\mu>\delta \mu\Big)<\exp(-\delta n\mu/3) \ \text{ for any } \delta\geq 1.
\end{equation}
We apply this inequality to
\begin{equation}\label{eq:Ixj}
X_{ij}=I(ij\in E\text{ and }\|\bt_i^*-\bt_j^*\|<1/c^*), \text{ where } i\in [n] \text{ is fixed and } j \in [n] \setminus \{i\}.
\end{equation}
As we explained above, $X_{ij}\sim $Bern$(\mu)$, where $\mu\leq pg(c^*)$ and thus with probability $1-\exp(-\Omega(\delta npg(c^*)))$
\begin{align*}
\sum_{j\in N(\bt_i^*)} I\Big(\|\bt_i^*-\bt_j^*\|<\frac{1}{c^*}\Big)=\sum_{j\in [n]\setminus \{i\}} X_{ij}\lesssim (\delta+1) npg(c^*)\approx\delta npg(c^*).
\end{align*}
Taking $\delta=p/(\log^4n g(c^*))$ results in
\begin{equation}\label{eq:caseAconclusion}
\sum_{j\in N(\bt_i^*)} I\Big(\|\bt_i^*-\bt_j^*\|<\frac{1}{c^*}\Big) \lesssim \frac{np^2}{\log^4 n} \ \text{ w.p. } 1-e^{-\Omega\left(\frac{np^2}{\log^4 n}\right)}.
\end{equation}
Note that the assumptions $g(c^*)\lesssim n^{-1/2}$ and $p\gtrsim \sqrt[3]{\log n/n}$ guarantee that our choice of $\delta$ satisfies the constraint $\delta \geq 1$ in \eqref{eq:chernoff2}.
Indeed, $\delta = p/(\log^4ng(c^*))=\Omega(n^{1/6}/\log^{11/3}n)>1$ for $n$ sufficiently large. Also, the assumption $p\gtrsim \sqrt[3]{\log n/n}$
implies that $\Omega(np^2/(\log^4 n))\gtrsim n^{1/3}/\log^{3/10} n$. Therefore, the probability in \eqref{eq:caseAconclusion} is greater than $1-O(n^{-6})$ and thus \eqref{eq:lessmax}
is proved in the current case.

\subsection{Proof for the case where $\boldsymbol{g(c^*)\gtrsim 1/\sqrt n}$.}\label{sec:thm4caseB}
We use another version of Chernoff bound~\cite{chernoff} for Bernoulli random variables: If $X_1, X_2,\cdots, X_n$ $\sim $Bern$(\mu)$ i.i.d., then
\begin{equation}\label{eq:chernoff}
\Pr\Big(\Big|\frac{1}{n}\sum_{i=1}^n X_i-\mu\Big|>\delta \mu\Big)<2 \cdot \exp(-\delta^2\mu n/3) \ \text{ for all } 0\leq \delta\leq 1.
\end{equation}
Applying this inequality to $\{X_{ij}\}_{j \in [n]\setminus \{i\}}$ of \eqref{eq:Ixj} yields that with probability $1-\exp(-\Omega(npg(c^*)))$
\begin{align}\label{eq:goalcaseB}
\sum_{j\in N(\bt_i^*)} I\Big(\|\bt_i^*-\bt_j^*\|<\frac{1}{c^*}\Big)=\sum_{j\in [n]\setminus \{i\}}X_{ij}
\lesssim npg(c^*).
\end{align}
Note that the probability $1-\exp(-\Omega(npg(c^*)))$ exponentially approaches $1$ as $n\to\infty$.
Indeed,  the assumptions $g(c^*)\gtrsim 1/\sqrt n$ and $p\gtrsim n^{-1/3}\log^{1/3} n$ imply that $\Omega(npg(c^*))=\Omega(n^{1/6}\log^{1/3}n)$.

Our goal is to upper bound the RHS of \eqref{eq:goalcaseB} by the second term in the RHS of \eqref{eq:lessmax}. In order to do this we use the following Lemmas, which we prove in Section \ref{sec:lemma}.
\begin{customlemma}{4}\label{thm:lemmahg}
Assuming the setting of Theorem \ref{thm:corruptions}, there exists an absolute constant $M$ such that
\begin{equation}\label{eq:hg}
\frac{1}{c^*}\leq M \ \text{ w.p. }1-O(n^{-6}).
\end{equation}
\end{customlemma}

\begin{customlemma}{5}\label{thm:lemmah}
Assume the setting of Theorem \ref{thm:corruptions}. If $g(c^*)\gtrsim 1/\sqrt n$, then
\begin{equation}\label{eq:q}
\frac{g(c^*)}{c^*} \lesssim \frac{\epsilon_b \sqrt{\log n}}{p} \ \text{ w.p. }1-O(n^{-6}).
\end{equation}

\end{customlemma}

Given the setting of Theorem \ref{thm:corruptions}, we claim that there exists $\bx_M\in \mathbb{R}^3$ with $\|\bx_M\|=M$ such that
\begin{align}\label{eq:gest}
\Phi(\bx_M)\text{Vol}\Big(\frac{1}{c^*}\Big)\leq g(c^*)\leq \Phi(\b0)\text{Vol}\Big(\frac{1}{c^*}\Big) \ \text{ w.p. } 1-O(n^{-6}),
\end{align}
where Vol$(r)$ is the volume of $B(0,r)$. The second inequality of \eqref{eq:gest} is deterministic and follows from the definition of $g$ in \eqref{eq:gc}. The first inequality follows from Lemma \ref{thm:lemmahg}. Indeed, with the same probability the minimum of $\Phi$ in the closed ball $\overline{B(\b0,1/c^*)}$ is greater than the minimum of  $\Phi$ in $\overline{B(\b0,M)}$ and it occurs on the boundary of this ball.
Equation \eqref{eq:gest} implies that $g(c^*)\approx 1/(c^*)^3$ and applying this observation to \eqref{eq:q} results in
\begin{equation}\label{eq:q2}
g(c^*) \lesssim \Big(\frac{\epsilon_b \sqrt{\log n}}{p}\Big)^{\frac{3}{4}} \ \text{ w.p. }1-O(n^{-6}).
\end{equation}
Combining \eqref{eq:q2} with \eqref{eq:goalcaseB} yields that with probability $1-O(n^{-6})$,
\[
\sum_{j\in N(\bt_i^*)} I\Big(\|\bt_i^*-\bt_j^*\|<\frac{1}{c^*}\Big) \lesssim npg(c^*)\lesssim np\Big(\frac{\epsilon_b \sqrt{\log n}}{p}\Big)^{\frac{3}{4}}= p^{\frac14}\epsilon_b^{\frac34}n\log^{\frac38}n.
\]
This concludes Theorem \ref{thm:corruptions}, though it remains to prove Lemmas \ref{thm:lemmahg} and \ref{thm:lemmah}.
\subsection{Proofs of Lemmas \ref{thm:lemmahg} and \ref{thm:lemmah}}\label{sec:lemma}
We first establish the following inequality, which is necessary for the proofs of both lemmas:
\begin{equation}\label{eq:cstarupper}
\sum_{ij\in E: \ \|\bt_i^*-\bt_j^*\|<\frac{1}{c^*}}\|\bt_i^*-\bt_j^*\|\lesssim\epsilon_b n^2\sqrt{\log n} \ \text{ w.p. } 1-O(n^{-6}).
\end{equation}
We prove \eqref{eq:cstarupper} by establishing an inequality involving the left and right derivatives of $f_{ij}(c\bt_i^*\,, c\bt_j^*)$ in $c$.
Since $f_{ij}(\bt_i,\bt_j)$ only depends on $\bt_i-\bt_j$ and since we assumed that $\bt_s=\b0$, $c^*$ can be defined as follows:
\begin{equation}\label{eq:oracle1}
c^*=\argmin_{c\in\mathbb{R}}\sum_{ij\in E}F_{ij}(c),
\end{equation}
where $F_{ij}(c)=f_{ij}(c\bt_i^*,c\bt_j^*)$.
This expression implies that
\begin{equation}\label{eq:opt}\sum_{ij\in E}F_{ij}'(c^{*-})\leq 0 \text{ and } \sum_{ij\in E}F_{ij}'(c^{*+})\geq 0.
\end{equation}
Indeed, WLOG if the second inequality in \eqref{eq:opt} is violated and  $\sum_{ij\in E}F_{ij}'(c^{*+}) < 0$, then there exists $\tilde c>c^*$ such that $\sum_{ij\in E}F_{ij}(\tilde c)<\sum_{ij\in E}F_{ij}(c^*)$. This contradicts the global optimality of $c^*$.

We estimate $F'_{ij}(c^+)$ for $ij\in E$ in 4 complementary cases.

\begin{enumerate}\item For $ij\in \Eg$ and $c\geq1/\|\bt_i^*-\bt_j^*\|$, $F_{ij}(c)=0$ and thus
$
F'_{ij}(c^+)=0.
$
\item For $ij\in \Eg$ and $c<1/\|\bt_i^*-\bt_j^*\|$, $F_{ij}(c)=1-\|\bt_i^*-\bt_j^*\| \cdot c$ and thus
$
F'_{ij}(c)=-\|\bt_i^*-\bt_j^*\|.
$
\item For $ij\in \Eb$ and $c\geq 1/\langle\bt_i^*-\bt_j^*\,, \bga_{ij}\rangle$, $F_{ij}(c)=\sin\alpha \cdot \|\bt_i^*-\bt_j^*\| \cdot c $, where $0<\alpha\leq\pi/2$ and thus
$
F'_{ij}(c^+)\leq \|\bt_i^*-\bt_j^*\|.
$
\item For $ij\in \Eb$ and $c<1/\langle\bt_i^*-\bt_j^*\,, \bga_{ij}\rangle$, $F_{ij}(c)=\|c\bt_i^*-c\bt_j^*-\bga_{ij}\| $ and thus by the triangle inequality
\begin{align*}
\left|F'_{ij}(c^+)\right|&=\lim_{h\to 0^+}\left|\frac{\|(c+h)\bt_i^*-(c+h)\bt_j^*-\bga_{ij}\| -\|c\bt_i^*-c\bt_j^*-\bga_{ij}\| }{h}\right|\\
&\leq\lim_{h\to 0^+}\left|\frac{\|h\bt_i^*-h\bt_j^*\| }{h}\right|= \|\bt_i^*-\bt_j^*\|.
\end{align*}
\end{enumerate}

The combination of the 4 cases above and the second inequality of \eqref{eq:opt} yield
 \begin{equation}\label{eq:condition2}
  -\sum_{ij\in \Eg: \ \|\bt_i^*-\bt_j^*\|<\frac{1}{c^*}}\|\bt_i^*-\bt_j^*\| + \sum_{ij\in \Eb}F_{ij}'(c^{*+})\geq 0.\end{equation}
Combining $|F'_{ij}(c^+)|\leq \|\bt_i^*-\bt_j^*\|$ with \eqref{eq:condition2}   results in the estimate
 \begin{align}
  \sum_{ij\in \Eg: \ \|\bt_i^*-\bt_j^*\|<\frac{1}{c^*}}\|\bt_i^*-\bt_j^*\| 
  \leq\sum_{ij\in \Eb}F_{ij}'(c^*)
 \leq \sum_{ij\in \Eb}\|\bt_i^*-\bt_j^*\|\leq \sum_{ij\in \Eb}(\|\bt_i^*\|+\|\bt_j^*\|) \lesssim  \epsilon_b n^2 \cdot \max_{i\in [n]}\|\bt_i^*\|.\label{eq:ebmax}
  \end{align}
By the second property of Lemma 3.10 of \cite{HandLV15} and its proof,
\begin{equation}\label{eq:maxti}
\max_{i\in [n]}\|\bt_i^*\|\lesssim\sqrt{\log n}\ \text{ w.p. } 1-O(n^{-6}).
\end{equation}
This observation and \eqref{eq:ebmax} results in \eqref{eq:cstarupper}.

Using \eqref{eq:cstarupper}, we prove Lemma \ref{thm:lemmahg} and \ref{thm:lemmah} in Sections \ref{sec:lemmahg} and \ref{sec:lemmah} respectively.

\subsubsection{Proof of Lemma \ref{thm:lemmahg}}\label{sec:lemmahg}
We assume on the contrary that $1/c^* > M$ and use this assumption to derive an inequality for the random variables
\begin{equation}
\label{eq:Yij}
Y_{ij}=I(ij\in E \text{ and }\|\bt_i^*-\bt_j^*\|< 1/c^*) \cdot \|\bt_i^*-\bt_j^*\| \ \text{ for fixed } i \in [n] \text{ and } j \in [n]\setminus \{i\}.
\end{equation}
This inequality uses
the constant $\mu_0=\inf_{\|\bx\|< 5}\mathbb{E}[I(\|\bx-\by\|< 1/c^*) \cdot \|\bx-\by\|]$, where $\by\sim N(\b0,\bI)$, and is formulated as follows:
\begin{align}\label{eq:lemmahobj}
\frac12 n^2p\mu_0&\lesssim \sum_{\latop{i \in [n]:}{\|\bt_i^*\|< 5}}\sum_{\ j\in [n]\setminus \{i\}}Y_{ij}\lesssim \frac{n^2p^{7/3}}{\log^{4}n}  \ \text{ w.p. } 1-O(n^{-6}).
\end{align}
We note that \eqref{eq:lemmahobj} results in contradiction w.p.~$1-O(n^{-6})$ and thus concludes the proof. Indeed, it implies that with this probability $\mu_0\lesssim p^{4/3}/\log^{4} n\to 0$ as $n\to \infty$. Since $\mu_0$ is monotonically increasing as a function of $1/c^*$, $1/c^*\to 0$ as $n\to \infty$, which contradicts our assumption.

The rest of this section proves \eqref{eq:lemmahobj} under the assumption that  $1/c^* > M$. We first establish the second inequality of \eqref{eq:lemmahobj} as follows.
We first note that
\begin{equation} \sum_{\latop{i \in [n]:}{\|\bt_i^*\| < 5}} \sum_{\ j \in [n] \setminus \{ i\}}  Y_{ij} \leq \sum_{i \in [n]} \sum_{\ j \in [n] \setminus \{ i\}} Y_{ij}  = 2 \sum_{ij \in E(K_n)} Y_{ij}.
\label{eq:prop_Yij}
\end{equation}
Subsequently applying \eqref{eq:prop_Yij}, the definition of $Y_{ij}$, \eqref{eq:cstarupper} and the assumption of Theorem \ref{thm:corruptions} that  $\epsilon_b= O(p^{7/3}/\log^{9/2} n)$, we obtain that
\begin{align}
\sum_{\latop{i \in [n]:}{\|\bt_i^*\| < 5}} \sum_{\ j\in [n]\setminus \{i\}} Y_{ij}\leq 2\sum_{ij\in E: \ \|\bt_i^*-\bt_j^*\|<\frac{1}{c^*}}\|\bt_i^*-\bt_j^*\|\lesssim\epsilon_b n^2\sqrt{\log n}\lesssim \frac{n^2p^{7/3}}{\log^{4}n}.
\end{align}

To prove the first inequality of \eqref{eq:lemmahobj}, we introduce the following notation:
Fix $i \in [n]$ and assume that $\|\bt_i^*\|< 5$.
Assume further that $\bt_1^*,\dots, \bt_n^*$ are  i.i.d.~$N(\b0,\bI)$ and let $Y_{ij}$ be defined in \eqref{eq:Yij},
$\bar Y_i= \sum_{j\in [n]\setminus \{i\}} Y_{ij}/(n-1)$ and $\mu_i=\mathbb{E}(\bar Y_i)=p \cdot \mathbb{E}[I(\|\bt_i^*-\bt_j^*\|< 1/c^*) \cdot \|\bt_i^*-\bt_j^*\|)]$.
Applying Hoeffding's inequality~\cite{hoeffding} to $\{Y_{ij}\}_{j \in [n] \setminus \{i\}}$
\begin{equation}\label{eq:meandev0}
\bar Y_i\geq \frac12 \mu_i \ \text{ w.p. } 1-2 \cdot \exp\left(-\frac{\mu_i^2n}{2\cdot\max\{Y_{ij}^2\}}\right).
\end{equation}
Since $\mu_i$ is monotonically increasing with respect to $1/c^*$, the assumption that $1/c^*>M$ implies that $\mu_i=\Omega(1)$. Combining this observation with \eqref{eq:meandev0} and  the definitions of $\mu_i$ and $\mu_0$ results in
\begin{equation}\label{eq:meandev}
\bar Y_i\geq \frac12 \mu_i\geq \frac12 \mu_0 p \ \text{ w.p. } 1-2 \cdot \exp\left(-\Omega\left(\frac{n}{\max\{Y_{ij}^2\}}\right)\right).
\end{equation}
Using the definition of $\bar Y_i$, we rewrite \eqref{eq:meandev} as follows: For fixed $i\in [n]$ with $\|\bt_i^*\|< 5$
\begin{equation}\label{eq:mulower}
\sum_{j\in [n]\setminus \{i\}}Y_{ij}\gtrsim np\mu_0 \ \text{ w.p. } 1-2 \cdot \exp\left(-\Omega\left(\frac{n}{\max\{Y_{ij}^2\}}\right)\right).
\end{equation}
A union bound of \eqref{eq:mulower} over all $i \in [n]$   with $\|\bt_i^*\|<5$ has the following form:
\begin{multline}\label{eq:unionyij}\sum_{\latop{i \in [n]:}{\|\bt_i^*\| < 5}}  \sum_{\ j\in [n]\setminus \{i\}}  Y_{ij}\gtrsim  \sum_{\latop{i \in [n]:}{\|\bt_i^*\| < 5}}  np\mu_0
= \sum_{i\in [n]} I(\|\bt_i^*\|<5)\cdot np\mu_0\\
\text{w.p. } 1-2\sum_{i\in [n]} I(\|\bt_i^*\|<5) \cdot\exp\left(-\Omega\left(\frac{n}{\max\{Y_{ij}^2\}}\right)\right).
\end{multline}

In order to conclude the first inequality of  \eqref{eq:lemmahobj} from \eqref{eq:mulower}, we first note that the application of \eqref{eq:chernoff} yields
\begin{equation}\label{eq:n2}
\sum_{i=1}^nI(\|\bt_i^*\|<5)>n/2 \ \text{ w.p. } 1-2 \cdot \exp(-\Omega(n)),
\end{equation}
and the application of basic inequalities and \eqref{eq:maxti} implies that
\begin{equation}\label{eq:boundxij}
0\leq \max_{ij\in E}\{Y_{ij}\}\leq \max_{ij\in E} \{\|\bt_i^*-\bt_j^*\|\}\leq 2\cdot\max_{i\in [n]}\{\|\bt_i^*\|\} \lesssim \sqrt{\log n} \ \text{ w.p. } 1-O(n^{-6}).
\end{equation}
Using \eqref{eq:n2}, we replace $\sum_{i=1}^nI(\|\bt_i^*\|<5)$ with $n/2$ in \eqref{eq:unionyij}. However, the new
probabilistic estimate is obtained by a union bound that uses the probabilities in \eqref{eq:n2} and \eqref{eq:unionyij}.
We thus obtain that
\begin{equation}\label{eq:eqyij}
\sum_{\latop{i \in [n]:}{\|\bt_i^*\| < 5}}\sum_{\ j\in [n]\setminus \{i\}}Y_{ij}\gtrsim \frac12n^2p\mu_0 \ \text{ w.p. } 1-n \cdot \exp\left(-\Omega\left(\frac{n}{\max\{Y_{ij}^2\}}\right)\right)-2 \cdot \exp(-\Omega(n)).
\end{equation}
Similarly, using \eqref{eq:boundxij}, we replace $\max\{Y_{ij}^2\}$ in the probability of \eqref{eq:eqyij} with $\log(n)$, but we also modify this probability by applying a union bound that uses the probabilities of \eqref{eq:eqyij} and \eqref{eq:boundxij}. We thus obtain that
\begin{equation*}
\sum_{\latop{i \in [n]:}{\|\bt_i^*\| < 5}}\sum_{\ j\in [n]\setminus \{i\}}Y_{ij}\gtrsim \frac12n^2p\mu_0 \ \text{ w.p. } 1-n \cdot \exp\Big(-\Omega\Big(\frac{n}{\log n}\Big)\Big)-2 \cdot \exp(-\Omega(n))-O(n^{-6}).
\end{equation*}
Note that this equation immediately implies \eqref{eq:lemmahobj} and thus concludes the proof of the lemma.

\subsubsection{Proof of Lemma \ref{thm:lemmah}}\label{sec:lemmah}
To prove the lemma, it suffices to verify w.p. $1-O(n^{-6})$ that
\begin{align}\label{eq:hcstarupper}
\sum_{ij\in E: \ \|\bt_i^*-\bt_j^*\|< \frac{1}{c^*}} \|\bt_i^*-\bt_j^*\|
\gtrsim \frac{1}{2c^*}\cdot npg(c^*) \cdot\frac{n}{2}.
\end{align}
Indeed, Lemma \ref{thm:lemmah} clearly follows by combining \eqref{eq:cstarupper} and \eqref{eq:hcstarupper}.

We first bound from below the LHS of \eqref{eq:hcstarupper} by a sum of random variables, which we define as follows. We arbitrarily fix  $i \in [n]$  such that $\|\bt_i^*\| < 5$ and for all $j \in [n]\setminus \{i\}$ let
$Z_{ij}=I(ij\in E \text{ and } 1/(2c^*)<\|\bt_i^*-\bt_j^*\|<1/c^*)$.
We note that
\begin{align}\label{eq:cstarzij}
\sum_{ij\in E: \ \|\bt_i^*-\bt_j^*\|< \frac{1}{c^*}} \|\bt_i^*-\bt_j^*\| &\geq \sum_{\latop{ij\in E: \ \|\bt_i^*\|< 5}{ \frac{1}{2c^*}<\|\bt_i^*-\bt_j^*\|<\frac{1}{c^*}}} \|\bt_i^*-\bt_j^*\|
=\frac{1}{2}\sum_{\latop{i \in [n]:}{\|\bt_i^*\|< 5}}\sum_{\latop{j\in N(\bt_i^*):} {\frac{1}{2c^*}<\|\bt_i^*-\bt_j^*\|<\frac{1}{c^*}}} \|\bt_i^*-\bt_j^*\|\nonumber\\
&=\frac{1}{2}\sum_{\latop{i \in [n]:}{\|\bt_i^*\|< 5}}\sum_{\ j\in [n]\setminus \{i\}} Z_{ij}\|\bt_i^*-\bt_j^*\|\geq \frac{1}{2}\cdot \frac{1}{2c^*}\sum_{\latop{i \in [n]:}{\|\bt_i^*\|< 5}}\sum_{\ j\in [n]\setminus \{i\}} Z_{ij}.
\end{align}

It remains to bound the RHS of \eqref{eq:cstarzij} by the RHS of \eqref{eq:hcstarupper} with high probability and conclude the proof. For this purpose, we introduce the following auxiliary function, which uses the random variable $\by \sim N(\b0,\bI)$,
\begin{equation}\label{eq:hc}
h(c^*)=\inf_{\|\bx\|<5}\Pr\Big(\Big\{\frac{1}{2c^*}<\|\bx-\by\|<\frac{1}{c^*}\Big\}\Big)=\inf_{\|\bx\|< 5}\int\limits_{B(\bx,\frac{1}{c^*})\setminus B(\bx,\frac{1}{2c^*})}\Phi(\bt)d\bt.
\end{equation}
In a somewhat similar way to establishing \eqref{eq:gest}, we note that there exists $\bx_0 \in \mathbb R^3$  with $\|\bx_0 \|=5$,
such that
\begin{align}
C_1\text{Vol}\Big(\frac{1}{2c^*}\Big)\leq h(c^*)&\leq C_2\text{Vol}\Big(\frac{1}{c^*}\Big) \ \text{ w.p. } 1-O(n^{-6}),\label{eq:hest}
\end{align}
where $C_1=\inf_{\|\bx-\bx_0\|< M} \Phi(\bx)$, $C_2=\sup_{\|\bx-\bx_0\|< M} \Phi(\bx)$. Thus, equation \eqref{eq:gest} and \eqref{eq:hest} imply that
\begin{equation}\label{eq:hgequiv}
g(c^*)\approx h(c^*)\approx \frac{1}{c^{*3}}   \ \text{ w.p. } 1-O(n^{-6}).
\end{equation}
We further note that $Z_{ij}\sim $Bern$(\mu_i)$, where $\mu_i\geq ph(c^*)$.  Combining this observation with \eqref{eq:chernoff} yields that
\begin{equation}\label{eq:nph}
\sum_{j\in [n]\setminus \{i\}}  Z_{ij} \gtrsim nph(c^*)\ \text{ w.p. }1-2 \cdot \exp(-\Omega(nph(c^*))).
\end{equation}

We conclude the proof of \eqref{eq:hcstarupper} as follows. Applying a union bound for \eqref{eq:nph} over all $i$ such that $\|\bt_i^*\|< 5$ yields
\begin{align}\label{eq:nph1}
\sum_{\latop{i \in [n]:}{\|\bt_i^*\|< 5}}\sum_{j\in [n] \setminus \{i\}} Z_{ij}
\gtrsim nph(c^*) \cdot \sum_{i=1}^nI(\|\bt_i^*\|<5)
\ \text{ w.p. } \
1-2 \sum_{i=1}^nI(\|\bt_i^*\|<5) \cdot \exp(-\Omega(np h(c^*))).
\end{align}
Using \eqref{eq:n2}, we replace $\sum_{i=1}^nI(\|\bt_i^*\|<5)$ with $n/2$ in \eqref{eq:nph1} and also modify the probabilistic estimate by a union bound that uses the probabilities in \eqref{eq:n2} and \eqref{eq:nph1} as follows
\begin{align}\label{eq:nph2}
\sum_{\latop{i \in [n]:}{\|\bt_i^*\|< 5}}\sum_{j\in [n] \setminus \{i\}} Z_{ij}
\gtrsim \frac{n^2}{2} ph(c^*)
\ \text{ w.p. } \
1-n \cdot \exp(-\Omega(np h(c^*)))-2 \cdot \exp(-\Omega(n)).
\end{align}
At last, by combining \eqref{eq:hgequiv} and \eqref{eq:nph2} and applying a union bound, we obtain that
\begin{align}
\sum_{\latop{i \in [n]:}{\|\bt_i^*\|< 5}}\sum_{j\in [n] \setminus \{i\}} Z_{ij}
\gtrsim  \frac{n^2}{2} p g(c^*) \label{eq:npgn}
\ \text{ w.p. } \
P_1=1-n \cdot \exp(-\Omega(np g(c^*)))-2 \cdot \exp(-\Omega(n))-O(n^{-6}).
\end{align}
The  assumptions $p=\Omega(\sqrt[3]{\log n/n})$ and $g(c^*)\gtrsim 1/\sqrt n$ imply that  $\Omega(npg(c^*))\gtrsim \Omega(n^{1/6}\log^{1/3} n)$. Therefore, $P_1=1-O(n^{-6})$.
Equation \eqref{eq:hcstarupper}, and thus the lemma, follows by
combining \eqref{eq:cstarzij} and \eqref{eq:npgn}.

\appendix

\section{Parallel Rigidity under the Setting of Theorem~\ref{thm:main}}
\label{sec:parallel}

A graph $G([n],E)$ with distinct vertex locations $\{\bt_i^*\}_{i=1}^n \subseteq \R^3$ and true edge directions $\{\bga_{ij}^*\}_{ij\in E}\in S^2$ is parallel rigid if its vertex locations can be uniquely recovered, up to scale and shift, from its edge directions.
Parallel rigidity was studied in graph theory \cite{matroid, pr_2006, pr_2009, pr_1989} and depends only on the graph $G([n],E)$ and the embedding dimension, which is 3 in our case. \"{O}zyesil, Singer and Basri~\cite{OzyesilSB15_SDR} noted its relevance for well-posedness of the camera location recovery problem. \"{O}zyesil and Singer \cite{cvprOzyesilS15} showed that it is sufficient for uniqueness of LUD when $|\Eb|=0$ (see Proposition 1 of \cite{cvprOzyesilS15}). We next show that parallel rigidity holds with overwhelming probability under the setting of Theorem \ref{thm:main}.

\begin{proposition}
\label{prop:para}
A graph $G([n],E)$ generated according to the setting of Theorem \ref{thm:main} is parallel rigid with overwhelming probability.
\end{proposition}

\begin{proof}
We use the following notation. For $S \subseteq [n]$, $E(S)=\{ij\in E:\,i,\,j\in S\}$ and for $i \in [n]$, $ \deg(i,S)= \sum_{j\in S} I(ij \in E)$.
For $E' \subseteq E$ and $i \in [n]$ denote $\deg(i, E')=\sum_{j\in [n]} I(ij\in E')$. Note that for $i \in S$, $\deg(i, E(S))=\deg(i,S)$.
For a node $k \in [n]$, $N_k$ denotes the set of neighbors of $k$. That is, $N_k$ includes all nodes that are connected to node $k$ by an edge.

Since $G([n],E)$ is $p$-typical, we may pick a node $k$ such that $\frac12 np \leq \deg(k,E)\leq 2 np$ and consequently $\frac12 np\leq |N_k|\leq 2np$. We first prove that  $G(N_k, E(N_k))$  is connected with overwhelming probability. The subgraph $G(N_k, E(N_k))$ is a realization of an Erd\"{o}s-R\'{e}nyi random graph $G(|N_k|, p)$ and it is connected with overwhelming probability. 
Indeed, for $1\leq m\leq |N_k|/2$,
\begin{align}
&\Pr(\exists \, m\text{ } \,\text{nodes that are isolated from the remaining nodes})\nonumber\\
\leq&\sum_{m=1}^{|N_k|/2} {|N_k|\choose m} (1-p)^{m(|N_k|-m)}\leq \sum_{m=1}^{|N_k|/2} {\left(\frac{e|N_k|}{m}\right)^m} e^{-pm(|N_k|-m)}\leq \frac{|N_k|}{2} \sup_{1\leq m\leq |N_k|/2}\left(\frac{e|N_k|}{e^{p(|N_k|-m)}}\right)^m \label{eq:stirling}\\
\nonumber
\leq & \frac{|N_k|}{2} \sup_{1\leq m\leq |N_k|/2}\left(\frac{e|N_k|}{e^{p|N_k|/2}}\right)^m\leq np \sup_{1\leq m\leq np}\left(\frac{2enp}{e^{np^2/4}}\right)^m \lesssim n^{4/3}\exp(-\Omega(n^{1/3}\log^{2/3} n))\label{eq:np}.
\end{align}
Note that the first inequality in \eqref{eq:stirling} uses a basic counting argument, where there are $m (|N_k|-m)$ possible edges between $m$ fixed elements and the remaining $|N_k|-m$ elements. The second inequality in \eqref{eq:stirling} follows from Stirling's approximation and the inequality $1-p \leq e^{-p}$. The last inequality in \eqref{eq:stirling} uses the assumption $p=\Omega(n^{-1/3}\log^{1/3}n)$.

Next, we prove that $G(\{k\}\cup N_k, E(\{k\}\cup N_k))$ is parallel rigid. Since $k$ is connected to all the vertices in $N_k$, and $E(N_k)$ forms a connected graph, the graph $G(\{k\}\cup N_k, E(\{k\}\cup N_k))$ can be generated by the following basic construction, which is similar to the Henneberg construction~\cite{henn} that preserves parallel rigidity at all of its steps. We start from a triangle $i_1i_2k$, where $i_1i_2\in E(N_k)$. By the connectivity of $G(N_k, E(N_k))$, there exists at least a vertex $i_3\in N_k$ that is connected to at least one of $i_1$ and $i_2$. WLOG we assume that $i_2i_3\in E(N_k)$ and thus $i_2$, $i_3$, $k$ form a triangle. Since the triangles $i_1i_2k$ and $i_2i_3k$ share the common edge $i_2k$, the graph $G(\{i_1,i_2,i_3,k\}, E(\{i_1,i_2,i_3,k\}))$ is parallel rigid. This procedure repeats by inductively adding vertices $i_4, i_5\dots i_{|N_k|}\in N_k$ to the existing graph. The graph $G(\{k\}\cup N_k, E(\{k\}\cup N_k))$, as well as each subgraph created in this procedure, are parallel rigid due to the following basic observation: If $G_1(V_1,E_1)$ and $G_1(V_2,E_2)$ are parallel rigid graphs and $E_1 \cap E_2 \neq \emptyset$, then $G(V_1 \cup V_2 ,E_1 \cup E_2)$ is parallel rigid.

At last, we prove that $G([n],E)$ is parallel rigid. Let $M_k = [n]\setminus (\{k\}\cup N_k)$.  By applying the Hoeffding's inequality, for any $l\in M_k$, $\deg(l, N_k)\geq \frac12 p|N_k|\geq \frac14 np^2$ with probability at least $1-\exp(-\Omega(np^2))$. By the assumption of Theorem \ref{thm:main} that $p\gtrsim n^{-1/3}\log^{1/3}n$ and by applying a union bound over $l\in M_k$, we obtain that $\min_{l\in M_k}\deg(l, N_k)\geq 2$ with overwhelming probability. Thus for any $l\in M_k$ there exists $i,j\in N_k$ such that $i$, $j$, $k$, $l$ form a quadrilateral that is parallel rigid in $\mathbb{R}^3$. Following the basic observation mentioned in proving the parallel rigidity of $G(\{k\}\cup N_k, E(\{k\}\cup N_k))$ and the fact that $i$, $j$, $k$ are already contained in the parallel rigid graph $G(\{k\}\cup N_k, E(\{k\}\cup N_k))$, we conclude that the graph $G(\{k, l\}\cup N_k, E(\{k, l\}\cup N_k))$ is parallel rigid.
By inductively adding vertices in $M_k$ in the same way, we obtain that the graph $G([n],E)=G(\{k\}\cup N_k\cup M_k, E(\{k\}\cup N_k\cup M_k))$ is parallel rigid.
\end{proof}

\section{On Uniqueness of LUD  and $c^*$}\label{sec:cstar}

In this section we show that under the setting of Theorem \ref{thm:main} with $|\Eb|>0$, the solution of LUD is unique with overwhelming probability.
Consequently, under this setting $c^*$ is uniquely determined with overwhelming probability.
Most of the discussion here assumes the deterministic setting mentioned earlier, though without assuming uniqueness of $c^*$.
The probabilistic setting only appears in Proposition \ref{thm:coro2}.

The following definition of self-consistency and non-self-consistency is essential in this section.
\begin{definition}
 Given any graph $G([n],E)$, a set of pairwise directions $\{\bga_{ij}\}_{ij\in E}\in S^2$ is self-consistent with respect to $G$
if there exist $\bt_1, \ldots, \bt_n \in \mathbb{R}^3$ that are not all identical such that $(\bt_i-\bt_j)=\|\bt_i-\bt_j\|\bga_{ij}$ for each $ij\in E$.
Otherwise $\{\bga_{ij}\}_{ij\in E}$ is non-self-consistent.
\end{definition}

Figure~\ref{fig:self} demonstrates an example of a graph with 3 vertices, where the corrupted pairwise directions are self-consistent and the locations obtained from them are different than the ground truth locations. This special example demonstrates a general phenomenon, which follows from the above definition. Whenever the corrupted edges are self-consistent, they give rise to a set of locations that are different than the ground truth locations. That is, non-self-consistency is a necessary condition for exact recovery when $|\Eb|>0$.

\begin{figure}[h!]
	\centering
	\includegraphics[width=.5\textwidth]{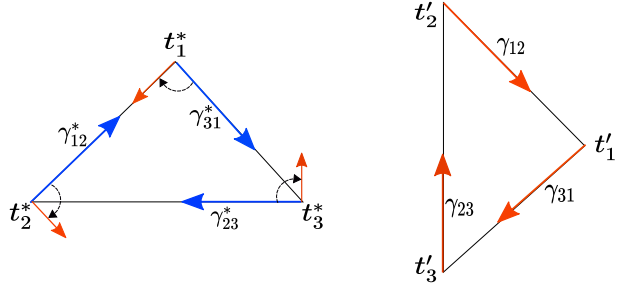}
	\caption{Demonstration of self-consistency. The figure on the left shows a graph with 3 vertices, ground truth locations $\bt^*_1$, $\bt^*_2$, $\bt^*_3$, ground truth pairwise directions $\bga^*_{21}$, $\bga^*_{32}$, $\bga^*_{13}$ and corrupted pairwise directions $\bga_{21}$, $\bga_{32}$, $\bga_{13}$. Note that the corrupted pairwise directions are obtained by $90$ degrees rotations of the ground truth ones.  The figure on the right shows a graph determined by the corrupted pairwise directions and its locations $\bt'_1$, $\bt'_2$, $\bt'_3$. Clearly, the latter locations are different than the ground truth ones for any arbitrary shift and scale.}
\label{fig:self}
\end{figure}

We next show that non-self-consistency is a sufficient condition for uniqueness of LUD in the corrupted case.
\begin{theorem}\label{thm:unique}
Given any graph $G([n],E)$ with non-self-consistent pairwise directions $\{\bga_{ij}\}_{ij\in E}$, the solution of LUD is unique.
\end{theorem}

\begin{proof}[Proof of Theorem~\ref{thm:unique}]

Assuming that $\{\bga_{ij}\}_{ij\in E}$ is non-self-consistent, we will show that any two solutions $(\{\hat\bt_i\}_{i=1}^n$, $\{\hat\alpha_{ij}\}_{ij \in E})$
and $(\{\bt_i'\}_{i=1}^n$, $\{\alpha'_{ij}\}_{ij \in E})$ of \eqref{eq:LUD} are the same.
For $0 \leq {\lambda} \leq 1$, define $\bt_i^{\lambda} = (1-{\lambda}) \hat\bt_i + {\lambda} \bt_i'$ and
$\alpha_{ij}^{\lambda} = (1-{\lambda}) \hat\alpha_{ij} + {\lambda} \alpha'_{ij}$. We note that since \eqref{eq:LUD} 
is a convex optimization problem,
for any $0 \leq {\lambda} \leq 1$, $(\{\bt_i^{\lambda}\}_{i=1}^n$, $\{\alpha_{ij}^{\lambda}\}_{ij \in E})$
is also a solution of \eqref{eq:LUD}. Therefore, the objective function evaluated at the solution $(\{\bt_i^{\lambda}\}_{i=1}^n$,
$\{\alpha_{ij}^{\lambda}\}_{ij \in E})$, namely $F({\lambda}) = \sum_{ij \in E} \|\bt_i^{\lambda}-\bt_j^{\lambda}-\alpha_{ij}^{\lambda} \bga_{ij}\|$, is constant on $[0,1]$.
We denote $\hat\be_{ij}=\hat\bt_i-\hat\bt_j-\hat\alpha_{ij}\bga_{ij}$ and $\be_{ij}'=\bt_i'-\bt_j'-\alpha_{ij}'\bga_{ij}$ and rewrite $F({\lambda})$ as
\begin{equation*}
F({\lambda})=\sum_{ij\in E}\|\hat\be_{ij} + {\lambda} (\be_{ij}'-\hat\be_{ij})\|= \sum_{ij\in E}\sqrt{\|\be_{ij}'-\hat\be_{ij}\|^2{\lambda}^2+2{\lambda}\hat\be_{ij}^T(\be_{ij}'-\hat\be_{ij})+\|\hat\be_{ij}\|^2}.
\end{equation*}
Since $F$ is constant, this equation implies that $\hat \be_{ij}=\be_{ij}'$ for all $ij\in E$. That is,
\begin{equation}\label{eq:sys}
\hat\bt_i-\hat\bt_j-\hat\alpha_{ij}\bga_{ij}=\bt_i'-\bt_j'-\alpha_{ij}'\bga_{ij} \ \text{ for } ij\in E.
\end{equation}
Let $\Delta \bt_i=\hat\bt_i-\bt_i'$ for $i\in [n]$ and $\Delta \alpha_{ij}=\hat\alpha_{ij}-\alpha_{ij}'$ for $ij \in E$. We rewrite \eqref{eq:sys} as
\begin{equation}\label{eq:sys1}
\Delta\bt_i-\Delta\bt_j=\Delta\alpha_{ij}\bga_{ij} \ \text{ for } ij\in E.
\end{equation}
Since $\|\bga_{ij}\|=1$, \eqref{eq:sys1} implies that
\begin{equation}\label{eq:sys2}
\Delta\bt_i-\Delta\bt_j=\|\Delta\bt_i-\Delta\bt_j\|\bga_{ij} \ \text{ for } ij\in E.
\end{equation}
The non-self-consistency of $\{\bga_{ij}\}_{ij \in E}$ implies that the elements of the solution $\{\Delta \bt_i\}_{i=1}^n$ of \eqref{eq:sys2} are all identical. Consequently, for all $i \in [n]$, $\hat \bt_i-\bt_i'$ is a constant vector in $\mathbb{R}^3$. The constraint $\sum_i\bt_i=\b0$ of \eqref{eq:LUD1} implies that the constant vector is zero and thus the solution is unique.
\end{proof}

Proposition \ref{thm:coro2} below guarantees  with overwhelming probability the non-self-consistency of $\{\bga_{ij}\}_{ij\in E}$  assuming both corruption and the setting of Theorem~\ref{thm:main}. Combined with Theorem~\ref{thm:unique}, it concludes the uniqueness of LUD in the corrupted case.
The proof of this result depends on Lemma \ref{thm:deg} below, which demonstrates a necessary condition for self-consistency. Before stating and proving these results, we introduce the following notation.

Let $G([n],E)$ be a graph, $T=\{\bt_i^*\}_{i=1}^n$ be a set of distinct vertex locations and assume that the assigned pairwise directions $\{\bga_{ij}\}_{ij\in E}$ is self-consistent and $\{\bga_{ij}\}_{ij\in E}\neq \{\bga_{ij}^*\}_{ij\in E}$.
As clarified above, $\{\bga_{ij}\}_{ij\in E}$ is the set of true pairwise directions of a set of locations $T' = \{\bt_i'\}_{i=1}^n \neq T$
and $T$ cannot be obtained from $T'$ by scaling and shifting.
      One may view $T'$ as perturbed vertices of $T$, even though the actual perturbation is of $\{\bga_{ij}\}_{ij\in E}$.
For $S \subset [n]$, denote  $T(S)=\{\bt_i^*\}_{i\in S}$ and $T'(S)=\{\bt_i'\}_{i\in S}$. We also use the notation $E(S)$, $\deg(i,S)$ and $\deg(i,E')$ (for $E' \subseteq E$), which was introduced in Appendix~\ref{sec:parallel} (see proof of Proposition \ref{prop:para}).
We say that $i$, $j\in [n]$ are undeformed and denote it by $i\sim j$, if $i\neq j$ and $\exists$ $\kappa>0$ such that $\bt_i^*-\bt_j^*= \kappa(\bt_i'-\bt_j')$.
Otherwise, we say that $i$ and $j$ are deformed and denote $i\nsim j$. Note that by definition $i \nsim i$. For each $i\in [n]$, we define the undeformed set $S_i=\{j\in [n]: j\sim i\}$.  The following lemma shows a critical property of self-consistent corruption. That is, for any self-consistent corruption of pairwise directions, there exists a vertex such that more than half of the rest of vertices are deformed with respect to it.
\begin{lemma}\label{thm:deg}
Let $G([n],E)$ be a graph and let $T=\{\bt_i^*\}_{i=1}^n$ be a set of distinct vertex locations. If the assigned pairwise directions $\{\bga_{ij}\}_{ij\in E}$ is self-consistent and $\{\bga_{ij}\}_{ij\in E}\neq \{\bga_{ij}^*\}_{ij\in E}$, then there exists $j\in [n]$  such that $|S_j|< n/2$.
\end{lemma}
\begin{proof}
Assume on the contrary that for all $j\in [n]$, $|S_j| \geq n/2$.
Since $|\Eb|\neq 0$, there exists $k, l \in [n]$ such that $k\nsim l$, which implies that
$\{k,l\}\cap (S_k\cup S_l)=\emptyset$ and $|S_k\cup S_l|\leq n-2$.
Consequently, $|S_k\cap S_l|=|S_k|+|S_l|-|S_k\cup S_l|\geq n/2+n/2-(n-2)=2$. Denote by $a$ and $b$ two of the elements of $S_k\cap S_l$ and note that by definition of the undeformed sets $S_k$ and $S_l$, $a \sim k$, $b \sim k$, $a \sim l$ and $b \sim l$. Due to the HLV model, the probability that $\{ak,bk,al,bl\}$ lies on a plane in $\mathbb{R}^3$ is zero and thus the graph $G(\{a,b,k,l\},\{ak,bk,al,bl\})$ is parallel rigid in $\mathbb{R}^3$~\cite[Figure 4(d)]{OzyesilSB15_SDR}. Therefore, $T(\{a,b,k,l\})=T'(\{a,b,k,l\})$ up to scale and shift and $k\sim l$, which results in contradiction.
\end{proof}

\begin{proposition}
\label{thm:coro2}
In the setting of Theorem \ref{thm:main}, if $|\Eb|\neq 0$, then $\{\bga_{ij}\}_{ij\in E}$ is non-self-consistent with overwhelming probability.
\end{proposition}
\begin{proof}
We show contradiction assuming that $\{\bga_{ij}\}_{ij\in E}$ is self-consistent. By Lemma \ref{thm:deg}, there exists $j\in [n]$  such that $|S_j|< n/2$. Note that $\deg(j, \Eb)=\deg(j, E(S_j^c))$. Therefore, $n\epsilon_b=\max_{i\in[n]}\deg(i, \Eb)\geq \deg(j,E(S_j^c))$. For each $i\in S_j^c\setminus \{j\}$,  $I(ij\in E(S_j^c))$ is a Bernoulli random variable Bern$(p)$. Thus, by applying \eqref{eq:chernoff} with $\delta=1/2$, $\mu=p$ and the number of terms $|S_j^c|-1=n-|S_j|-1>n/2-1$, we obtain that
\begin{equation}\label{eq:EAdeg}
\deg(j,E(S_j^c))=\sum_{i\in S_j^c \setminus\{j\}} I\left(ij\in E(S_j^c)\right)> \frac12 \cdot (\frac{n}{2}-1)p \ \text{ w.p. } 1-2e^{-\frac{1}{12}(\frac{n}{2}-1)p}.
\end{equation}
Combining the assumption $p=\Omega(n^{-1/3}\log^{1/3}n)$ with \eqref{eq:EAdeg} implies that $n\epsilon_b\geq \deg(j,E(S_j^c))=\Omega(np)$ with probability $1-2 \cdot \exp(-\Omega(n^{2/3}\log^{1/3}n)$.
This contradicts the assumption of Theorem \ref{thm:main} that $n\epsilon_b=O(np^{7/3}/\log^{9/2}n)$.
\end{proof}


\bibliographystyle{abbrv}
\bibliography{refs_prop_16}

\end{document}

%% file: rrp-macro-file.tex
\usepackage{amsthm}
\usepackage{amsmath}
\usepackage{amssymb}
\usepackage{bm}
\usepackage{mathrsfs}
\usepackage[dvips]{graphicx}

\usepackage{algorithmic}
\usepackage{algorithm}

\usepackage{color}

\usepackage{url}

\usepackage{supertabular}

\theoremstyle{definition}

\theoremstyle{remark}

\numberwithin{thm}{section}

\DeclareMathAlphabet{\mathsfsl}{OT1}{cmss}{m}{sl}

\renewcommand{\phi}{\varphi}

\newcommand{\latop}[2]{\genfrac{}{}{0pt}{}{#1}{#2}}

\newcommand{\R}{\mathbb{R}}

\newcommand{\argmin}{\operatorname*{arg\,min}}

\newcommand{\trace}{\operatorname{tr}}

\newcommand{\bx}{\boldsymbol{x}}

\newcommand{\bT}{\boldsymbol{T}}

\newcommand{\by}{\boldsymbol{y}}

\newcommand{\bSigma}{\boldsymbol\Sigma}

\def\Eb{E_{{b}}}
\def\Eg{E_{{g}}}
\def\Egl{E_{{gl}}}

\def\reals{\mathbb{R}}
\def\bx{\boldsymbol{x}}
\def\bga{\boldsymbol{\gamma}}
\def\beps{\boldsymbol{\epsilon}}
\def\be{\boldsymbol{e}}
\def\bt{\boldsymbol{t}}

\def\b0{\boldsymbol{0}}

\def\bSigma{\boldsymbol\Sigma}

\def\bv{\boldsymbol{v}}

\def\bh{\boldsymbol{h}}

\def\bI{\boldsymbol{I}}

%% file: LUD_arxiv_revised.bbl
\begin{thebibliography}{10}

\bibitem{Nachimson_LS}
M.~Arie{-}Nachimson, S.~Z. Kovalsky, I.~Kemelmacher{-}Shlizerman, A.~Singer,
  and R.~Basri.
\newblock Global motion estimation from point matches.
\newblock In {\em 2012 Second International Conference on 3D Imaging, Modeling,
  Processing, Visualization {\&} Transmission, Zurich, Switzerland, October
  13-15, 2012}, pages 81--88, 2012.

\bibitem{BrandAT04_LS}
M.~Brand, M.~E. Antone, and S.~J. Teller.
\newblock Spectral solution of large-scale extrinsic camera calibration as a
  graph embedding problem.
\newblock In {\em Computer Vision - {ECCV} 2004, 8th European Conference on
  Computer Vision, Prague, Czech Republic, May 11-14, 2004. Proceedings, Part
  {II}}, pages 262--273, 2004.

\bibitem{CandesLMW11_robustpca}
E.~J. Cand{\`{e}}s, X.~Li, Y.~Ma, and J.~Wright.
\newblock Robust principal component analysis?
\newblock {\em J. {ACM}}, 58(3):11:1--11:37, 2011.

\bibitem{CandesT05_decode}
E.~J. Cand{\`{e}}s and T.~Tao.
\newblock Decoding by linear programming.
\newblock {\em {IEEE} Trans. Information Theory}, 51(12):4203--4215, 2005.

\bibitem{ChandrasekaranSPW11}
V.~Chandrasekaran, S.~Sanghavi, P.~A. Parrilo, and A.~S. Willsky.
\newblock Rank-sparsity incoherence for matrix decomposition.
\newblock {\em {SIAM} Journal on Optimization}, 21(2):572--596, 2011.

\bibitem{ChatterjeeG13_rotation}
A.~Chatterjee and V.~M. Govindu.
\newblock Efficient and robust large-scale rotation averaging.
\newblock In {\em {IEEE} International Conference on Computer Vision, {ICCV}
  2013, Sydney, Australia, December 1-8, 2013}, pages 521--528, 2013.

\bibitem{CoudronL12_reaper}
M.~Coudron and G.~Lerman.
\newblock On the sample complexity of robust {PCA}.
\newblock In {\em Advances in Neural Information Processing Systems 25: 26th
  Annual Conference on Neural Information Processing Systems 2012. Proceedings
  of a meeting held December 3-6, 2012, Lake Tahoe, Nevada, United States.},
  pages 3230--3238, 2012.

\bibitem{matroid}
M.~Develin, J.~L. Martin, and V.~Reiner.
\newblock Rigidity theory for matroids.
\newblock {\em Comment. Math. Helv.}, 82(1):197--233, 2007.

\bibitem{pr_2006}
T.~Eren, W.~Whiteley, and P.~N. Belhumeur.
\newblock Using angle of arrival (bearing) information in network localization.
\newblock In {\em Decision and Control, 2006 45th IEEE Conference on}, pages
  4676--4681. IEEE, 2006.

\bibitem{GoldsteinHLVS16_shapekick}
T.~Goldstein, P.~Hand, C.~Lee, V.~Voroninski, and S.~Soatto.
\newblock Shapefit and shapekick for robust, scalable structure from motion.
\newblock In {\em Computer Vision - {ECCV} 2016 - 14th European Conference,
  Amsterdam, The Netherlands, October 11-14, 2016, Proceedings, Part {VII}},
  pages 289--304, 2016.

\bibitem{Govindu01_LS}
V.~M. Govindu.
\newblock Combining two-view constraints for motion estimation.
\newblock In {\em 2001 {IEEE} Computer Society Conference on Computer Vision
  and Pattern Recognition {(CVPR} 2001), 8-14 December 2001, Kauai, HI, {USA}},
  pages 218--225, 2001.

\bibitem{Govindu04_Lie}
V.~M. Govindu.
\newblock Lie-algebraic averaging for globally consistent motion estimation.
\newblock In {\em 2004 {IEEE} Computer Society Conference on Computer Vision
  and Pattern Recognition {(CVPR} 2004), 27 June - 2 July 2004, Washington, DC,
  {USA}}, pages 684--691, 2004.

\bibitem{henn}
R.~Haas, D.~Orden, G.~Rote, F.~Santos, B.~Servatius, H.~Servatius, D.~Souvaine,
  I.~Streinu, and W.~Whiteley.
\newblock Planar minimally rigid graphs and pseudo-triangulations.
\newblock {\em Computational Geometry}, 31(1-2):31--61, 2005.

\bibitem{HandLV15}
P.~Hand, C.~Lee, and V.~Voroninski.
\newblock Shapefit: Exact location recovery from corrupted pairwise directions.
\newblock {\em Communications on Pure and Applied Mathematics}, 71(1):3--50,
  2018.

\bibitem{multiviewbook}
A.~Harltey and A.~Zisserman.
\newblock {\em Multiple view geometry in computer vision {(2.} ed.)}.
\newblock Cambridge University Press, 2006.

\bibitem{HartleyAT11_rotation}
R.~I. Hartley, K.~Aftab, and J.~Trumpf.
\newblock {L1} rotation averaging using the {W}eiszfeld algorithm.
\newblock In {\em The 24th {IEEE} Conference on Computer Vision and Pattern
  Recognition, {CVPR} 2011, Colorado Springs, CO, USA, 20-25 June 2011}, pages
  3041--3048, 2011.

\bibitem{hoeffding}
W.~Hoeffding.
\newblock Probability inequalities for sums of bounded random variables.
\newblock {\em J. Amer. Statist. Assoc.}, 58(301):13--30, 1963.

\bibitem{pr_2009}
B.~Jackson and T.~Jord{\'a}n.
\newblock Graph theoretic techniques in the analysis of uniquely localizable
  sensor networks.
\newblock In {\em Localization Algorithms and Strategies for Wireless Sensor
  Networks: Monitoring and Surveillance Techniques for Target Tracking}, pages
  146--173. IGI Global, 2009.

\bibitem{LMTZ2014}
G.~Lerman, M.~B. McCoy, J.~A. Tropp, and T.~Zhang.
\newblock Robust computation of linear models by convex relaxation.
\newblock {\em Foundations of Computational Mathematics}, 15(2):363--410, 2015.

\bibitem{MartinecP07_rotation}
D.~Martinec and T.~Pajdla.
\newblock Robust rotation and translation estimation in multiview
  reconstruction.
\newblock In {\em 2007 {IEEE} Computer Society Conference on Computer Vision
  and Pattern Recognition {(CVPR} 2007), 18-23 June 2007, Minneapolis,
  Minnesota, {USA}}, 2007.

\bibitem{chernoff}
M.~Mitzenmacher and E.~Upfal.
\newblock {\em Probability and computing: Randomized algorithms and
  probabilistic analysis}.
\newblock Cambridge university press, 2005.

\bibitem{MoulonMM13_Linfty}
P.~Moulon, P.~Monasse, and R.~Marlet.
\newblock Global fusion of relative motions for robust, accurate and scalable
  structure from motion.
\newblock In {\em {IEEE} International Conference on Computer Vision, {ICCV}
  2013, Sydney, Australia, December 1-8, 2013}, pages 3248--3255, 2013.

\bibitem{cvprOzyesilS15}
O.~{\"{O}}zyesil and A.~Singer.
\newblock Robust camera location estimation by convex programming.
\newblock In {\em {IEEE} Conference on Computer Vision and Pattern Recognition,
  {CVPR} 2015, Boston, MA, USA, June 7-12, 2015}, pages 2674--2683, 2015.

\bibitem{OzyesilSB15_SDR}
O.~{\"{O}}zyesil, A.~Singer, and R.~Basri.
\newblock Stable camera motion estimation using convex programming.
\newblock {\em {SIAM} Journal on Imaging Sciences}, 8(2):1220--1262, 2015.

\bibitem{sfmsurvey_2017}
O.~\"{O}zyesil, V.~Voroninski, R.~Basri, and A.~Singer.
\newblock A survey of structure from motion.
\newblock {\em Acta Numerica}, 26:305--364, 2017.

\bibitem{ravikumar2011}
P.~Ravikumar, M.~J. Wainwright, G.~Raskutti, and B.~Yu.
\newblock High-dimensional covariance estimation by minimizing
  $\ell_1$-penalized log-determinant divergence.
\newblock {\em Electron. J. Statist.}, 5:935--980, 2011.

\bibitem{SenguptaAGGJSB17}
S.~Sengupta, T.~Amir, M.~Galun, T.~Goldstein, D.~W. Jacobs, A.~Singer, and
  R.~Basri.
\newblock A new rank constraint on multi-view fundamental matrices, and its
  application to camera location recovery.
\newblock {\em {IEEE} Conference on Computer Vision and Pattern Recognition,
  {CVPR} 2017, Honolulu, Hawaii, USA, June 22-25, 2017}, pages 4798--4806,
  2017.

\bibitem{AAB}
Y.~Shi and G.~Lerman.
\newblock Estimation of camera locations in highly corrupted scenarios: All
  about that base, no shape trouble.
\newblock In {\em Proceedings of the IEEE Conference on Computer Vision and
  Pattern Recognition}, pages 2868--2876, 2018.

\bibitem{TronV09_CLS1}
R.~Tron and R.~Vidal.
\newblock Distributed image-based 3-d localization of camera sensor networks.
\newblock In {\em Proceedings of the 48th {IEEE} Conference on Decision and
  Control, {CDC} 2009, December 16-18, 2009, Shanghai, China}, pages 901--908,
  2009.

\bibitem{TronV14_CLS2}
R.~Tron and R.~Vidal.
\newblock Distributed 3-d localization of camera sensor networks from 2-d image
  measurements.
\newblock {\em {IEEE} Trans. Automat. Contr.}, 59(12):3325--3340, 2014.

\bibitem{pr_1989}
W.~Whiteley.
\newblock A matroid on hypergraphs, with applications in scene analysis and
  geometry.
\newblock {\em Discrete \& Computational Geometry}, 4(1):75--95, 1989.

\bibitem{1dsfm}
K.~Wilson and N.~Snavely.
\newblock Robust global translations with 1dsfm.
\newblock In {\em Computer Vision - {ECCV} 2014 - 13th European Conference,
  Zurich, Switzerland, September 6-12, 2014, Proceedings, Part {III}}, pages
  61--75, 2014.

\bibitem{XuCS12_robustpca}
H.~Xu, C.~Caramanis, and S.~Sanghavi.
\newblock Robust {PCA} via outlier pursuit.
\newblock {\em {IEEE} Trans. Information Theory}, 58(5):3047--3064, 2012.

\bibitem{ZhangL14_novel}
T.~Zhang and G.~Lerman.
\newblock A novel {M}-estimator for robust {PCA}.
\newblock {\em Journal of Machine Learning Research}, 15(1):749--808, 2014.

\bibitem{BATA}
B.~Zhuang, L.-F. Cheong, and G.~H. Lee.
\newblock Baseline desensitizing in translation averaging.
\newblock In {\em Proceedings of the IEEE Conference on Computer Vision and
  Pattern Recognition}, pages 4539--4547, 2018.

\end{thebibliography}
